\documentclass[10pt,twocolumn,letterpaper]{article}

\usepackage{iccv}
\usepackage{times}
\usepackage{epsfig}
\usepackage{graphicx}
\usepackage{amsmath}
\usepackage{amssymb}
\usepackage{times}
\usepackage{epsfig}
\usepackage{graphicx}
\usepackage{amsmath}
\usepackage{amssymb}
\usepackage{amsthm}
\usepackage{caption}
\usepackage{xcolor}
\usepackage{subcaption}
\usepackage{array}
\usepackage{mathtools}
\usepackage[ruled,vlined]{algorithm2e}
\usepackage[pagebackref=true,breaklinks=true,letterpaper=true,colorlinks,bookmarks=false]{hyperref}

\iccvfinalcopy % *** Uncomment this line for the final submission

\newcommand{\figonespace}{2.8cm}

\newcommand{\bu}{\mathbf{u}}
\newcommand{\bx}{\mathbf{x}}

\newcommand{\bc}{\mathbf{c}}

\newcommand{\eps}{\varepsilon}
\newcommand{\R}{\mathbb{R}}

\newcommand{\boldstart}[1]{\vspace{0.06in}\noindent{\bf #1}}

\newtheorem{theorem}{Theorem}

\newcommand{\bbjparams}{\Theta}               % All junction parameters
\newcommand{\bjparams}{\boldsymbol{\theta}}   % M color junction parameters for a single patch
\newcommand{\anglesymbol}{\phi}               % Single junction angle
\newcommand{\banglesymbol}{\boldsymbol{\phi}} % Collection of all M junction angles
\newcommand{\distfunc}{d}                     % Distance functions
\newcommand{\bbjparamc}{\text{C}}             % All color functions
\newcommand{\bjparamc}{\bc}                   % M color functions for a single patch
\newcommand{\jparamc}{c}                      % Single color function for a single region
\newcommand{\colorfamily}{\mathcal{C}}        % Family of color functions
\newcommand{\channels}{K}                     % Number of color channels in image

\begin{document}

%%%%%%%%% TITLE
\title{Field of Junctions: Extracting Boundary Structure at Low SNR}

\author{Dor Verbin and Todd Zickler\\
Harvard University\\
{\tt\small \{dorverbin,zickler\}@seas.harvard.edu}
}

% https://tex.stackexchange.com/questions/55764/input-a-figure-between-title-and-body-in-twocolumn-form

\twocolumn[{%
\renewcommand\twocolumn[1][]{#1}%
\maketitle
\begin{center}
\centering
\captionsetup{type=figure}
\includegraphics[width=\linewidth]{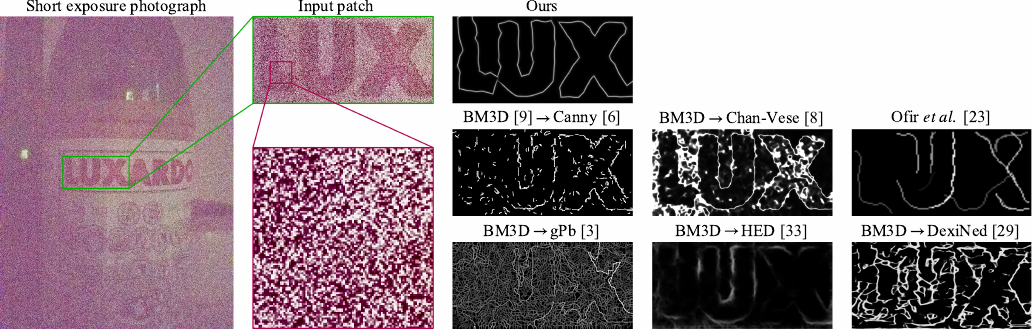}%\vspace*{-2.5mm}
\captionof{figure}{\label{figure:lux} Detecting boundaries at short exposure (1/5000s). The field of junctions extracts boundary structure at noise levels where other methods fail, even when the others are preceded by denoising and are optimally tuned for the image. Additionally, our model interprets its boundaries into component contours, corners, junctions, and regional colors (see Figure~\ref{figure:R}).}
\end{center}%
}]

%%%%%%%%% ABSTRACT
\begin{abstract}
We introduce a bottom-up model for simultaneously finding many boundary elements in an image, including contours, corners and junctions. The model explains boundary shape in each small patch using a `generalized $M$-junction' comprising $M$ angles and a freely-moving vertex. Images are analyzed using non-convex optimization to cooperatively find $M+2$ junction values at every location, with spatial consistency being enforced by a novel regularizer that reduces curvature while preserving corners and junctions. The resulting `field of junctions' is simultaneously a contour detector, corner/junction detector, and boundary-aware smoothing of regional appearance. Notably, its unified analysis of contours, corners, junctions and uniform regions allows it to succeed at high noise levels, where other methods for segmentation and boundary detection fail.
\end{abstract}

\section{Introduction} \label{section:introduction}

Identifying  boundaries is fundamental to vision, and being able to do it from the bottom up is helpful because vision systems are not always familiar with the objects and scenes they encounter. The essence of boundaries is easy to articulate: They are predominantly smooth and curvilinear; they include a small but important set of zero-dimensional events like corners and junctions; and in between
boundaries, regional appearance is homogeneous in some sense. 

\begin{figure*}[t]
\small
\hspace{10mm}
\makebox[0pt]{Input} \hfill
\makebox[0pt]{gPb [3]} \hfill
\makebox[0pt]{$\text{L}_0$ [34], ASJ [36]} \hfill
\makebox[0pt]{BM3D [9]$\rightarrow$gPb} \hfill
\makebox[0pt]{BM3D$\rightarrow$$\text{L}_0$,ASJ}
\hspace{13mm}
$\overbracket[1pt][1.0mm]{\hspace{47mm}}_%
    {\substack{\vspace{-6.0mm}\\
\colorbox{white}{Ours}}}$
\vspace{-4.5mm}
\begin{center}
    \includegraphics[trim=0 0 0 10,clip,width=\linewidth]{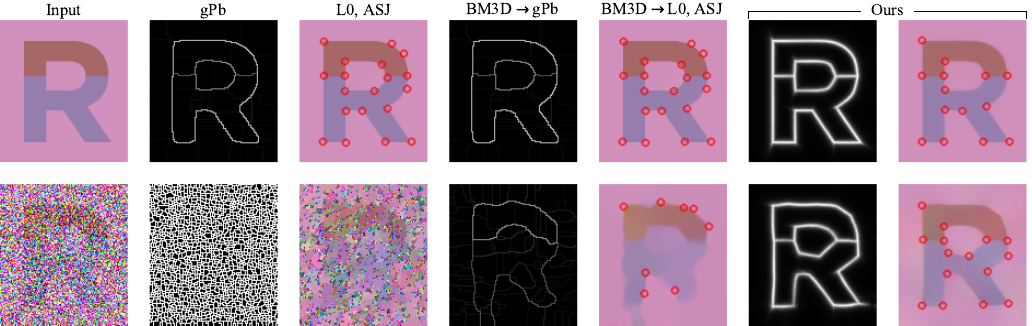}
\end{center}
\vspace*{-5mm}
\caption{Interpreting boundary structure at high and low SNR (top and bottom). The field of junctions identifies contours (column 6), corners/junctions (circles, column 7) and smooth colors (column 7). It is more resilient to noise than previous methods that are specific to contours, junctions or smoothing, even when they are preceded by optimally-tuned denoising.}\label{figure:R}
\end{figure*}

Yet, despite this succinct description, extracting boundaries that include all of these elements and exploit their interdependence has proven difficult. After decades of work on various subsets of contour detection, corner detection, junction detection, and segmentation, the community is still searching for comprehensive and reliable solutions. Even deep encoder-decoder CNNs, which can be tuned to exploit many kinds of local and non-local patterns in a dataset, struggle to localize boundaries with precision, motivating an ongoing search for architectural innovations like skip connections, gated convolutions, bilateral regularization, multi-scale supervision, kernel predictors, and so on.

We introduce a bottom-up model that precisely discerns complete boundary structure---contours, corners, and junctions---all at the same time (see Figures~\ref{figure:lux} \& \ref{figure:R}). It does this by fitting a non-linear representation to each small image patch, with $M+2$ values that explain the patch as being uniform or containing an edge, thin bar, corner, or junction of any degree up to $M$ (see Figure~\ref{figure:junctions}). The model encourages consistency between overlapping patches using a new form of spatial regularization that, instead of penalizing overall curve length or elastica, expresses preference for global boundary maps comprising isolated corners and junctions that are connected by contours with small curvature. As far as we know, this is the first time such regularization has been achieved in the presence of junctions.

An image is analyzed by solving a non-convex optimization problem that cooperatively determines $M+2$ junction values at every location. This produces a \emph{field of junctions}: a distilled representation of the contours, corners, junctions and homogeneous regions of an image. It is an intermediate representation that is useful for a variety of tasks, including contour detection, junction/keypoint detection, and boundary-aware smoothing. 

Experimentally, the field of junctions provides unprecedented resilience to noise. It is repeatable over a wide range of noise levels, including very high noise regimes where other approaches---whether based on denoising, segmentation, contour detection, or junction detection---all tend to fail (see Figures~\ref{figure:lux} \& \ref{figure:R}). We attribute this to the form of its regularization and to its unified representation of contours, corners, junctions and uniformity, which allows all of these signals to mutually excite and inhibit during analysis.

We introduce the field of junctions model in Section~\ref{section:model}, where we formulate analysis as a non-convex optimization problem. We describe how the model can be used for both single-channel and multi-channel images, and how it includes a parameter controlling the scale of its output. The following Section~\ref{section:analysis} is the heart of the paper: It introduces the optimization techniques that allow analysis to succeed. In particular, we present a greedy algorithm for initializing each patch's junction parameters that has convergence guarantees under certain conditions, and is very effective in practice even when they do not hold. In Section~\ref{section:experiments} we apply the field of junctions to contour, corner, and junction detection, showing that it provides novel regularization capabilities and repeatable performance across many noise levels. Extended versions of our figures, generalizations of the model, additional results, and a video summary of our paper, are all available in the supplement. 

\section{Related Work} \label{section:relatedwork}

\boldstart{Contour, corner and junction detection.} These have been studied for decades, often separately, using halved receptive fields to localize contours~\cite{canny1986computational,iverson1995logical,martin2004learning} and wedges or other patch-based models for corners and junctions~\cite{harris1988combined,rohr1992recognizing,deriche1993recovering,lowe2004distinctive,cazorla2003two,xia2014accurate,xue2017anisotropic}. The drawback of separating these processes is that, unlike our model, it does not exploit concurrency between contours, corners and junctions at detection time. 

\boldstart{Contour detection at low SNR.} The naive way to detect contours at low SNR is to precede a contour detector by a strong generic denoiser. Ofir~\etal~\cite{ofir2016fast,ofir2019detection} were perhaps the first to convincingly show that better results can be achieved by designing optimization strategies that specifically exploit the regularity of contours (also see Figure~\ref{figure:lux}). We build on this idea by developing different optimization schemes that handle a broader set of boundary structures and that improve upon~\cite{ofir2016fast,ofir2019detection} in both accuracy and scalability.

\boldstart{Curvature regularization.} Boundaries extracted at low SNR are strongly influenced by the choice of regularization. Prior work has shown that minimizing curvature---either alone or in combination with length (Euler's elastica)---generally does better at preserving elongated structures and fine details than minimizing length alone; and  there have been many attempts to invent good numerical schemes for minimizing boundary curvature~\cite{schoenemann2012linear,nieuwenhuis2014efficient,zhu2012image,zhong2020minimizing,tai2011fast,he2019segmentation}. All of these methods lead to rounded corners, and more critically, they only apply to boundaries between two regions so provide no means for preserving junctions (see Figure~\ref{figure:regularization}). 
In contrast, our model preserves sharp corners and junctions while also reducing curvature along contours.

\boldstart{Segmentation.} Our patch model is inspired by the level-set method of Chan and Vese~\cite{chan2001active} and in particular its multi-phase generalizations~\cite{vese2002multiphase,hodneland2009four}. In fact, our descent strategy in Section~\ref{section:GD} can be interpreted as pursuing optimal level-set functions in each patch, with each patch's functions constrained to a continuous $(M+2)$-parameter family. Our experiments show that our regularized patch-wise approach obviates the needs for manual initialization and re-initializing during optimization, both of which have been frequent requirements in practice~\cite{chan2001active,vese2002multiphase,hodneland2009four,li2005level}. 

\boldstart{Boundary-aware smoothing.} When locating boundaries, our model infers the regional colors adjacent to each boundary point and so provides boundary-aware smoothing as a by-product. It is not competitive with the efficiency of dedicated smoothers~\cite{xu2011image,paris2009fast,gastal2011domain} but is more resilient to noise.

\boldstart{Deep encoder/decoder networks}. Our approach is very different from relying on deep CNNs to infer the locations of boundaries~(\eg,~\cite{xie2015holistically,soria2020dexined}) or lines and junctions~\cite{huang2018learning,zhou2019end,xue2019learning}. CNNs have an advantage of being trainable over large datasets, allowing both local and non-local patterns to be internalized and exploited for prediction; but there are ongoing challenges related to overcoming their internal spatial subsampling (which makes boundaries hard to localize) and their limited interpretability (which makes it hard to adapt to radically new situations). Unlike CNNs, the field of junctions model does not have capacity to maximally exploit the intricacies of a particular dataset or imaging modality. But it has the advantages of: not being subsampled; interpreting boundary structure into component contours, corners and junctions; applying to many noise levels and many single-channel or multi-channel 2D imaging modalities; and being controlled by just a few intuitive parameters.

\section{Field of Junctions}\label{section:model}

From a $K$-channel image $I\colon\Omega\rightarrow\R^{\channels}$ with 2D support $\Omega$, we extract dense, overlapping $R\times R$ spatial patches, denoted $\mathcal{I}_R = \{I_i(\bx)\}_{i=1}^N$. We also define a continuous family of patch-types,  $\mathcal{P}_R = \{\bu_{\bjparams}(\bx)\}$, parametrized by $\bjparams$, describing the boundary structure in an $R\times R$ patch. For $\mathcal{P}_R$ we use the family of \emph{generalized $M$-junctions}, comprising $M$ angular wedges around a vertex. The parameters $\bjparams = (\banglesymbol, \bx^{(0)}) \in \R^{M+2}$ are $M$ angles $\banglesymbol = (\anglesymbol^{(1)}, ..., \anglesymbol^{(M)})$ and vertex position $\bx^{(0)} = (x^{(0)}, y^{(0)})$. Importantly, the vertex can be inside or outside of the patch, and wedges may have size $0$. Figure~\ref{figure:junctions} shows examples for $M=3$.

Assume all image patches $\mathcal{I}_R$ are described by patches from $\mathcal{P}_R$ with additive white Gaussian noise. This means that for every $i \in \{1, ..., N\}$ there exist parameters $\bjparams_i$, and $M$ color functions $\jparamc_i^{(1)}, ..., \jparamc_i^{(M)}\colon\Omega_i \rightarrow \R^{\channels}$ (to be defined momentarily), such that for all $\bx \in \Omega_i$:
\begin{equation} \label{equation:forwardmodel}
I_i(\bx) = \sum_{j=1}^M  u^{(j)}_{\bjparams_i}(\bx)\jparamc_i^{(j)}(\bx) + n_i(\bx),
\end{equation}
where $n_i(\bx) \sim \mathcal{N}(0, \sigma^2)$ is noise, and $u^{(j)}_{\bjparams_i}\colon \Omega_i\rightarrow\{0, 1\}$ is an indicator function that returns $1$ if $\bx$ is inside the $j$th wedge defined by $\bjparams_i$ and $0$ otherwise.

Each color function $\jparamc_i^{(j)}$ is defined over the support of the $i$th patch $\Omega_i$ and explains the continuous field of $K$-channel values within the $j$th wedge of that patch. These functions are constrained to a pre-chosen family of functions $\colorfamily$, such as constant functions $\colorfamily = \{\jparamc(\bx) \equiv c \colon c \in \R^{\channels}\}$ or linear functions  $\colorfamily = \{\jparamc(\bx) = A \bx + b \colon A\in \R^{\channels\times 2}, b\in \R^{\channels}\}$.

\begin{figure}[t]
\begin{center}
    \includegraphics[width=\linewidth]{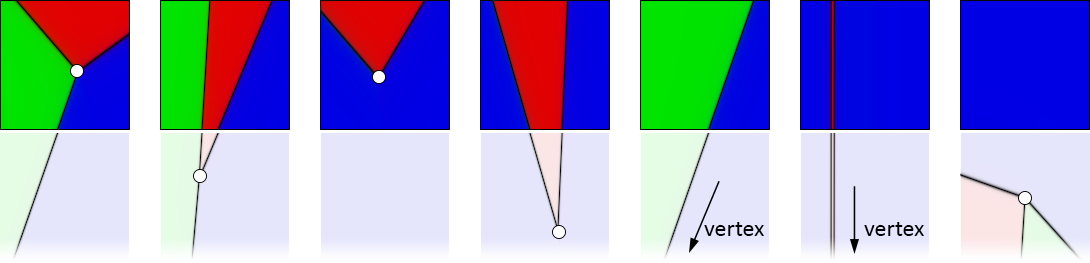}
\end{center}
\vspace*{-5mm}
\caption{A generalized $M$-junction comprises a vertex and $M$ angles, partitioning each patch into at most $M$ uniform regions (here, $M=3$). By freeing the vertex to be variously inside or outside of patches as needed, the model simultaneously accommodates contours, lines, corners, junctions, and uniform regions, thereby allowing concurrencies between all of them to be exploited during analysis.}\label{figure:junctions}
\end{figure}

We write the process of analyzing an image into its field of junctions as solving the optimization problem:
\begin{equation} \label{equation:MAP}
    \underset{\bbjparams, \bbjparamc}{\max} \log p(\bbjparams) + \log p(\bbjparamc) + \sum_{i=1}^N \log p(I_i | \bjparams_i, \bjparamc_i),
\end{equation}
where $p(\bbjparams)$ and $p(\bbjparamc)$ are spatial consistency terms over all junction parameters $\bbjparams = (\bjparams_1, ..., \bjparams_N)$ and color functions $\bbjparamc = (\bjparamc_1, ..., \bjparamc_N)$ respectively, and $p(I_i | \bjparams_i, \bjparamc_i)$ is the likelihood of a patch $I_i$ given the junction parameters $\bjparams_i$ and color functions $\bjparamc_i=(\jparamc_i^{(1)}, ..., \jparamc_i^{(M)})$. If the consistency terms $p(\bbjparams)$ and $p(\bbjparamc)$ are $0$ whenever overlapping patches disagree within their overlap, this objective is precisely the MAP estimate of the field of junctions, where the consistency terms are interpreted as priors over junction parameters and color functions, which we model as independent.

In the remainder of this section we provide more information about the three terms in Equation~\ref{equation:MAP}. For simplicity we use $M=3$ and a constant color model $c^{(j)}_i(\bx)\equiv c^{(j)}_i$, but expansions to higher-order color models and to $M>3$ are trivial and described in the supplement. The supplement also shows how the model performs when noise is not spatially-independent as is assumed in Equation~\ref{equation:forwardmodel}.

\subsection{Patch Likelihood} \label{section:dataterm}

For a single patch,  Equation~\ref{equation:forwardmodel} directly shows that the log-likelihood term is negatively proportional to the mean squared error in that patch:
\begin{equation} \label{equation:MSE2}
    \hskip-32px  % Fix equation number skipping to the next line
    \log p(I_i|\bjparams_i, \bjparamc_i) = -\alpha \sum_{j=1}^M \int u^{(j)}_{\bjparams_i}(\bx) \left\|\jparamc_i^{(j)} - I_i(\bx)\right\|^2 d\bx,
    \hskip-15px  % Fix equation number skipping to the next line
\end{equation}
where $\alpha > 0$ is a constant determined by the noise level $\sigma$.

The likelihood term in Equation~\ref{equation:MSE2} can be treated as a function of the junction parameters at a single location, $\bjparams_i$, because finding the optimal colors $\bjparamc_i$ is trivial for a given $\bjparams_i$ (see Equation~\ref{equation:colors} in Section~\ref{section:consistency}). However, despite the low dimensionality of the problem, which requires estimating an $(M+2)$-dimensional junction parameter per patch, solving it efficiently is a substantial challenge. We present an efficient solution to this problem in Section~\ref{section:CD}.

\begin{figure}[t]
\begin{center}
\smaller
\hspace{8.5mm}
\makebox[0pt]{Input} \hfill
\makebox[0pt]{$\ell_1$-Elastica [13]} \hfill
\makebox[0pt]{Ours, $\lambda_B=0.1$} \hfill
\makebox[0pt]{Ours, $\lambda_B=10$}
\hspace{9mm}

\includegraphics[trim=0 0 0 10,clip,width=\linewidth]{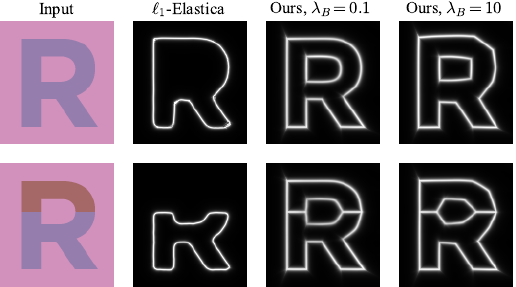}
\end{center}
\vspace*{-5mm}
\caption{Our boundary consistency term, governed by $\lambda_B$, favors isolated corners and junctions connected by contours with low curvature. Unlike other regularizers, it: is agnostic to contour length and convexity; preserves sharp corners; and preserves junctions between three or more regions.} \label{figure:regularization}
\end{figure}

\subsection{Spatial Consistency} \label{section:consistency}

Our spatial consistency terms $p(\bbjparams)$ and $p(\bbjparamc)$ require that all junction models agree within their overlap. The boundary consistency can be succinctly written as a constraint on the boundaries defined by each junction:
\begin{align} \label{equation:consistency01}
    \log p(\bbjparams) = \begin{cases} 0  &\text{ if } B_i(\bx) = \hat{B}(\bx) \text{ for all } $i$ \\ -\infty &\text{ otherwise}\end{cases},
\end{align}
where $B_i(\bx)$ is the \emph{boundary map} at the $i$th patch that returns $1$ if $\bx$ is a boundary location according to $\bjparams_i$ and $0$ otherwise, and $\hat{B}(\bx) = \max_{i \in \{1, ..., N\}} B_i(\bx)$ is the \emph{global boundary map} defined by the field of junctions.

The boundary consistency term in Equation~\ref{equation:consistency01} provides a hard constraint on the junction parameters, which is difficult to use in practice. We instead replace it with a relaxed, finite version having width $\delta$ and strength $\beta_B$:
\begin{equation} \label{equation:consistencysmooth}
    \log p(\bbjparams) = -\beta_B \sum_{i=1}^N \int \left[B^{(\delta)}_i(\bx) - \hat{B}^{(\delta)}_i(\bx)\right]^2 d\bx,
\end{equation}
where $B^{(\delta)}_i(\bx)$ is a smooth boundary map with dropoff width $\delta$ from the exact boundary position, to be defined precisely in Section~\ref{section:GD}. The relaxed global boundary map $\hat{B}^{(\delta)}(\bx)$ is now computed by taking the mean (rather than maximum) of the smooth local boundary map at each position $\bx$ over all patches containing it:
\begin{equation} \label{equation:edgemapsmooth}
    \hat{B}^{(\delta)}(\bx) = \frac{1}{|N_\bx|}\sum_{i \in N_\bx} B^{(\delta)}_i(\bx),
\end{equation}
where $N_\bx = \{i\colon \bx \in \Omega_i\}$ is the set of indices of patches that contain $\bx$. We denote by $\hat{B}^{(\delta)}_i(\bx)$ the $i$th patch of the relaxed global boundary map in Equation~\ref{equation:edgemapsmooth}. Note that the relaxed consistency in Equation~\ref{equation:consistencysmooth} approaches the strict one from Equation~\ref{equation:consistency01} when $\delta \rightarrow 0$ and $\beta_B \rightarrow \infty$.

Similar to the boundary spatial consistency term, we define the color spatial consistency term as:
\begin{equation} \label{equation:colorconsistencysmooth}
    \log p(\bbjparamc) = -\beta_C \sum_{i=1}^N \sum_{j=1}^M \int u_{\bjparams_i}^{(j)}(\bx) \left\|\jparamc_i^{(j)} - \hat{I}_i(\bx)\right\|^2 d\bx,
\end{equation}
where $\hat{I}_i(\bx)$ is the $i$th patch of the \emph{global color map}:
\begin{equation} \label{equation:colormapsmooth}
    \hat{I}(\bx) = \frac{1}{|N_\bx|}\sum_{i \in N_\bx}\sum_{j=1}^M u_{\bjparams_i}^{(j)}(\bx) \jparamc_i^{(j)}.
\end{equation}

Using the expressions for the log-likelihood and the relaxed consistency in  Equations~\ref{equation:MSE2},~\ref{equation:consistencysmooth}, and~\ref{equation:colorconsistencysmooth}, analyzing an image into its field of junctions can now be written as the solution to the following minimization problem:
\begin{align} \label{program:MAPsmooth}
    \underset{\bbjparams, \bbjparamc}{\min} &\sum_{i=1}^N \sum_{j=1}^M \int u^{(j)}_{\bjparams_i}(\bx) \left\|\jparamc_i^{(j)} - I_i(\bx) \right\|^2 d\bx  \nonumber\\
    &+\lambda_B \sum_{i=1}^N \int \left[B^{(\delta)}_i(\bx) - \hat{B}^{(\delta)}_i(\bx)\right]^2 d\bx, \\
    &+\lambda_C \sum_{i=1}^N \sum_{j=1}^M \int u_{\bjparams_i}^{(j)}(\bx) \left\|\jparamc_i^{(j)} - \hat{I}_i(\bx)\right\|^2 d\bx, \nonumber
\end{align}
where $\lambda_B = \beta_B / \alpha$ and $\lambda_C = \beta_C / \alpha$ are parameters controlling the strength of the boundary and color consistency.

We solve Problem~(\ref{program:MAPsmooth}) by alternation, updating junction parameters and colors $(\bbjparams ,\bbjparamc)$ while global maps $(\hat{B}^{(\delta)}, \hat{I})$ are fixed, and then updating the global maps. This takes advantage of closed-form expressions for the optimal colors. For the constant color model the expression is \begin{equation} \label{equation:colors}
    \jparamc_i^{(j)} = \frac{\int u^{(j)}_{\bjparams_i}(\bx) \left[I_i(\bx) + \lambda_C \hat{I}_i(\bx)\right]d\bx}{(1+\lambda_C)\int u^{(j)}_{\bjparams_i}(\bx)d\bx},
\end{equation}
and for piecewise-linear colors, \ie, $c_i^{(j)}(\bx) = A_i^{(j)}\bx + b_i^{(j)}$, there is a similar expression that replaces each of the $\channels$ divisions with a $3\times 3$ matrix inversion and multiplication.

Our formulation of boundary consistency encourages each patch $i$ to agree with its overlapping neighbors, by inhibiting its own boundariness $B_i^{(\delta)}(\bx)$ at pixels $\bx$ that are assigned a low score by their neighbors (as quantified by $\hat{B}^{(\delta)}(\bx)$) and exciting its boundariness at pixels assigned a high score. This means only salient junctions, corners, and contours end up contributing to the final global boundary map $\hat{B}^{(\delta)}(\bx)$. Junction values in uniform patches and other less salient patches tend to disagree with other patches, so spurious boundaries within them are suppressed. 

At the same time, our use of a smoothed version of consistency instead of a strict one allows for contours having nonzero curvature to be well approximated by local collections of corners that have slightly different vertices, while incurring a penalty. This has the effect of a curvature regularizer, because the only way for all junctions in the field to exactly agree is when the global boundary has zero curvature everywhere except at a finite number of vertices spaced at least $\ell_\infty$-distance $R$ apart (\eg a polygon).

The color consistency term of our objective promotes agreement on color between overlapping patches. It improves the results of the field of junctions under high noise by enforcing long-range consistency between the colors of sets of pixels not separated by a boundary.

\begin{figure}[t]
\begin{center}
    \includegraphics[width=\linewidth]{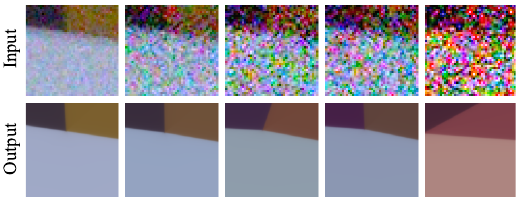}
\end{center}
\vspace*{-5mm}
\caption{Output of Algorithm~\ref{alg:greedy2} for a patch in SIDD~\cite{abdelhamed2018high} captured at decreasing light levels. Algorithm convergence is only guaranteed when noise is absent, but output is quite accurate in practice even when noise is high.} \label{figure:algorithm2}
\end{figure}

\section{Analysis}\label{section:analysis}

Analyzing an image into its field of junctions is a challenge, with Problem~(\ref{program:MAPsmooth}) consisting of $N$ junction-fitting problems that are coupled by spatial consistency terms. Even without consistency, finding the optimal junction for a single patch $i$ requires minimizing a non-smooth and non-convex function in $\bjparams_i$. 

We solve the problem in two parts: initialization and refinement. Both of these are key to our model's robustness to noise. The initialization procedure independently optimizes each patch, using a handful of coordinate updates to find discrete values for its angles and vertex location. Then, the refinement procedure performs gradient descent on a relaxation of Problem~(\ref{program:MAPsmooth}), cooperatively adjusting all junction parameters to find continuous angles and sub-pixel vertex locations that improve spatial consistency while maintaining fidelity to local appearance. We next describe each step.

\subsection{Initialization} \label{section:CD}

Many previous methods for junction estimation, such as~\cite{deriche1993recovering,cazorla2003two}, use gradient descent to optimize the vertex and angles of a single wedge model. These methods rely on having a good initialization from a human or a corner detector, and they fail when such initializations are unavailable. Indeed, even in the noiseless case, there always exists an initialization of a patch's junction parameters around which the negative log-likelihood is locally constant. 

In the present case, we need an initialization strategy that is automatic and reliable for \emph{every} patch, or at least the vast majority of them. We first describe an initialization algorithm for the simpler problem in which the vertex of a patch is known, where our algorithm guarantees optimality in the absence of noise; and then we expand it to solve for the vertex and angles together.

When the vertex is known, optimizing the parameters of one patch reduces to finding a piecewise-constant, one-dimensional angular function. There are algorithms for this based on dynamic programming~\cite{auger1989algorithms,jackson2005algorithm} and heuristic particle swarm optimization~\cite{bergerhoff2019algorithms}. We instead propose Algorithm~\ref{alg:greedy}, which is guaranteed to find the true junction angles $\banglesymbol = (\anglesymbol^{(1)}, ..., \anglesymbol^{(M)})$ that minimize the negative log-likelihood $\ell(\banglesymbol, \bx^{(0)}) = -\log p(I_i | \bjparams, \bjparamc_i)$ in the noiseless case. The algorithm consists of a single coordinate-descent update over the $M$ junction angles, that is, it minimizes $\ell_j(\anglesymbol) \overset{\Delta}{=}\nolinebreak \ell(\anglesymbol^{(1)}, ..., \anglesymbol^{(j-1)}, \anglesymbol, \anglesymbol^{(j+1)}, ..., \anglesymbol^{(M)}, x^{(0)}, y^{(0)})$ for $j = 1, ..., M$.

\begin{algorithm}
\SetAlgoLined
 Initialize $\anglesymbol^{(1)}, ..., \anglesymbol^{(M)} \leftarrow 0$.\\
\For{$j = 1, ..., M$}{
   $\anglesymbol^{(j)} \leftarrow \underset{\anglesymbol}{\text{argmin }}\ell_j(\anglesymbol)$
}
\caption{Optimization of angles}
\label{alg:greedy}
\end{algorithm}

\begin{figure*}
\begin{center}
\hspace{\figonespace}
\makebox[0pt]{Input}\hfill
\makebox[0pt]{Boundary-aware smoothing}\hfill
\makebox[0pt]{Boundaries}
\hspace{\figonespace}
% need newline here for some reason

\includegraphics[width=\linewidth]{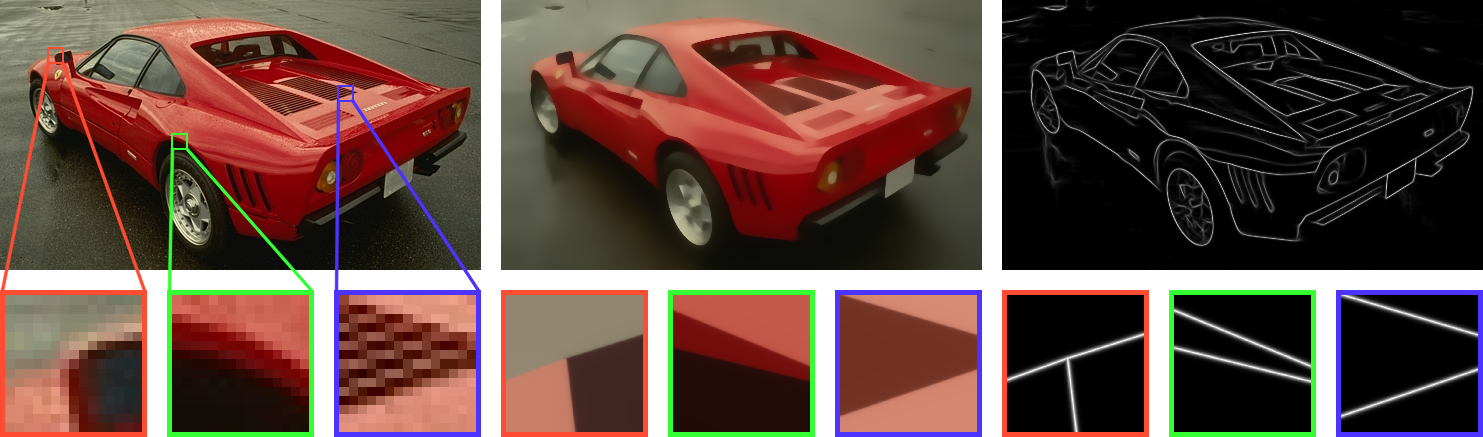}%\vspace*{-3mm}
\captionof{figure}{\label{figure:car} Field of junctions from a photograph. It can extract boundary-aware smoothing and boundary structure from natural images because it is robust to texture and other natural deviations from the ideal generalized $M$-junction model.}
\end{center}
\end{figure*}

\begin{theorem} \label{theorem:theorem}
For a junction image $I_i(\bx)$ with no noise (i.e., $n_i\equiv 0$ in Eq.~\ref{equation:forwardmodel}) and with vertex $\bx^{(0)}$ known, Algorithm~\ref{alg:greedy} is guaranteed to find the globally optimal angles $\banglesymbol$.
\end{theorem}

\begin{proof}[Proof Sketch]  (See full proof in supplement.)
First, note that $\ell_j(\anglesymbol)$ is continuous and smooth for all $\anglesymbol$ other than possibly a discontinuity in the derivative at any of the true junction angles. If the optimal $\anglesymbol$ is not one of the true junction angles then it must lie in the \emph{open} interval between two such angles, \ie $\anglesymbol \in (\anglesymbol^-, \anglesymbol^+)$. It can be shown that $\ell_j(\anglesymbol)$ does not have any local minima in $(\anglesymbol^-, \anglesymbol^+)$, and therefore for each angular interval between two true junction angles the cost function must be minimized at one of the endpoints. Therefore repeatedly minimizing $\ell_j(\anglesymbol)$ for $j = 1,..., M$ is guaranteed to provide a globally optimal set of angles.
\end{proof}

In practice, we find that Algorithm~\ref{alg:greedy} provides an excellent estimate of the true junction angles even when the input patch is noisy. It also has a significant efficiency advantage. Each coordinate update can be done to an arbitrarily small error $\eps$ with complexity $O(1/\eps)$, by exhaustively searching over all angles in increments of $\eps$. The complexity for a single junction is therefore $O(M/\eps)$, in contrast with the $O(1/\eps^2)$ dynamic programming solution of~\cite{jackson2005algorithm} and the $O(1/\eps^M)$ of naive exhaustive search over all possible $M$-angle sets. Moreover, each step of the algorithm can be run in parallel over all angles (and over all patches) by computing the value of $\ell_j(\anglesymbol^{(j)})$ for each of the $O(1/\eps)$ values and choosing the minimizing angle. Thus, runtime can be accelerated significantly using a GPU or multiple processors. 

These efficiency advantages become especially important when we expand the problem to optimize the vertex in addition to the angles. We simply do this by initializing the vertex at the center of the patch and updating it along with the angles using a coordinate descent procedure. See Algorithm~\ref{alg:greedy2}. Figure~\ref{figure:algorithm2} shows a typical example, where the algorithm results in a good estimate of the true vertex position and angles despite a substantial amount of noise.

\begin{algorithm}
\SetAlgoLined
 Initialize $x^{(0)}, y^{(0)}$ at the center of the patch.\\
\For{$i = 1, ..., N_{\text{init}}$}{
    Find angles $\banglesymbol$ using Algorithm~\ref{alg:greedy}.\\
    $x^{(0)} \leftarrow \underset{x}{\text{argmin }}\ell(\banglesymbol, x, y^{(0)})$\\
    $y^{(0)} \leftarrow \underset{y}{\text{argmin }}\ell(\banglesymbol, x^{(0)}, y)$\\
}
\caption{Optimization of angles and vertex}
\label{alg:greedy2}
\end{algorithm}

\subsection{Refinement} \label{section:GD}

After initializing each patch separately, we refine the field of junctions using continuous, gradient-based optimization. In order to compute the gradient of the objective in Problem~(\ref{program:MAPsmooth}) with respect to $\bbjparams$ we relax the indicator functions $\{\bu_{\bjparams}(\bx)\}$, making them smooth in $\bx$ and in $\bjparams$, similar to level-set methods~\cite{chan2001active,vese2002multiphase}. We do this by describing each $3$-junction using two distance functions (a similar parametrization exists using $M-1$ functions for $M$-junctions). Given the vertex position $(x^{(0)}, y^{(0)})$ and angles $\anglesymbol^{(1)}, \anglesymbol^{(2)}, \anglesymbol^{(3)}$, and assuming without loss of generality that $0 \leq \anglesymbol^{(1)} \leq \anglesymbol^{(2)} \leq \anglesymbol^{(3)} < 2\pi$, we define a junction using two signed distance functions $\distfunc_{12}$ and $\distfunc_{13}$ defined by:
\begin{equation} 
    \distfunc_{kl}(\bx) = \begin{cases}
    \min\{\distfunc_k(\bx), -\distfunc_l(\bx)\} \;\, \text{ if } \anglesymbol^{(l)} - \anglesymbol^{(k)} < \pi \\
    \max\{\distfunc_k(\bx), -\distfunc_l(\bx)\} \; \text{ otherwise} \end{cases}  
    \hskip-14px  % Fix equation number skipping to the next line
\end{equation}
where $\distfunc_l(x, y) = -(x-x^{(0)})\sin(\anglesymbol^{(l)}) + (y-y^{(0)})\cos(\anglesymbol^{(l)})$ is the signed distance function from a line with angle $\anglesymbol^{(l)}$ passing through $(x^{(0)}, y^{(0)})$.

Our relaxed indicator functions are defined as:
\begin{align}
    u_{\bjparams}^{(1)}(\bx) &= 1-H_{\eta}(\distfunc_{13}(\bx)), \nonumber \\
    u_{\bjparams}^{(2)}(\bx) &= H_{\eta}(\distfunc_{13}(\bx))\left[1-H_{\eta}(\distfunc_{12}(\bx))\right], \\
    u_{\bjparams}^{(3)}(\bx) &= H_{\eta}(\distfunc_{13}(\bx))H_{\eta}(\distfunc_{12}(\bx)), \nonumber
\end{align}
where $H_{\eta}$ is the regularized Heaviside function, as in~\cite{chan2001active}:
\begin{equation}
    H_{\eta}(d) = \frac{1}{2}\left[1 + \frac{2}{\pi}\arctan\left(\frac{d}{\eta}\right) \right].
\end{equation}
The smooth boundary maps for the consistency term are:
\begin{equation} \label{equation:edgemapsmoothparam}
    B_i^{(\delta)}(\bx) = \pi\delta\cdot H'_{\delta}\left(\min\{|d_{12}(\bx)|, |d_{13}(\bx)|\}\right),
\end{equation}
where $H'_{\delta}(d)$ is the derivative of $H_{\delta}(d)$ with respect to $d$, and the scaling factor ensures that $0 \leq B_i^{(\delta)}(\bx)\leq 1$. 

Our experiments use $\eta = 0.01$ and $\delta = 0.1$. We find that the algorithm is fairly insensitive to these values, and that varying them does not provide useful control of the model's behavior. This is in contrast to the other parameters---patch size $R$ and consistency weights $\lambda_B$, $\lambda_C$---that control scale and level of boundary and color detail.

\begin{figure}[t]
\begin{center}
    \includegraphics[width=\linewidth]{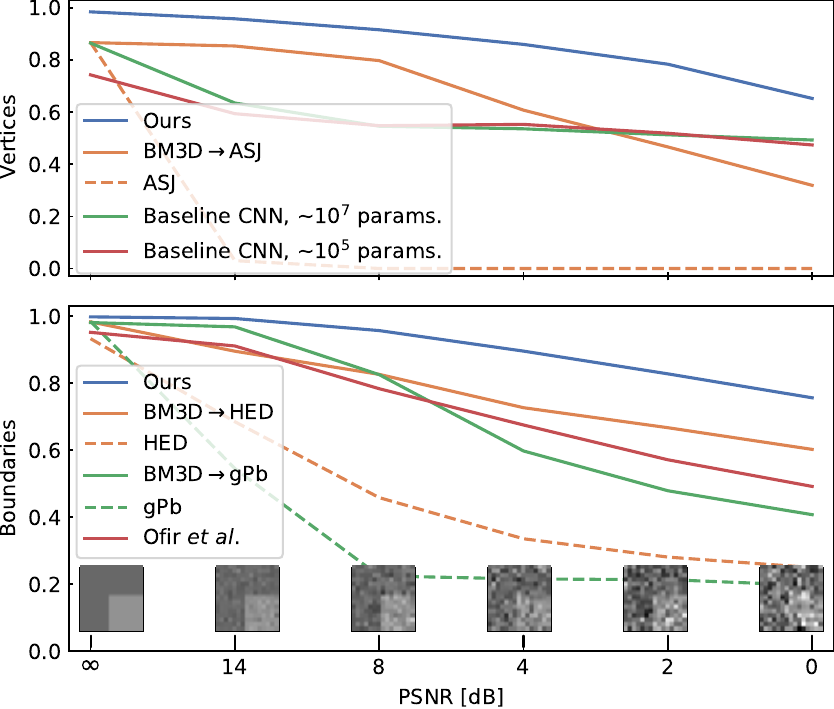}
\end{center}
\vspace*{-5mm}
\caption{Vertex and boundary detection F-score for increasing noise on our dataset. At low noise our model is comparable to existing edge and junction detectors and a baseline CNN, but it significantly outperforms them at high noise, even when preprocessed by BM3D. Insets: sample patch at different noise levels.} \label{figure:fmeasures}
\end{figure}

\subsection{Optimization Details}

We analyze an image into its field of junctions by first initializing with Algorithm~\ref{alg:greedy2} for $N_\text{init} = 30$ iterations, followed by refinement to minimize Problem~(\ref{program:MAPsmooth}) using the Adam optimizer~\cite{kingma2014adam} for $N_{\text{iter}} = 1000$ iterations. Initialization is performed by evaluating the restricted negative log-likelihood functions in Algorithms~\ref{alg:greedy} and~\ref{alg:greedy2} at $100$ evenly-spaced values. Because the vertex of a junction can be outside its patch (see Figure~\ref{figure:junctions}), each of its two coordinates is searched over an interval of length $3R$ around the center of each patch. The accuracy of our initialization is thus $3.6^\circ$ in the junction angles, and $0.03R$ in the vertex position.

For the refinement step we use a learning rate of $0.03$ for the vertex positions and $0.003$ for the junction angles, and the global maps $\hat{B}^{(\delta)}(\bx)$ and $\hat{I}(\bx)$ are treated as constants computed using the values of the previous iteration when computing gradients. In order to allow the parameters to first improve their estimates locally and only then use the consistency term to improve the field of junctions, we linearly increase the consistency weights from $0$ to their final values $\lambda_B$ and $\lambda_C$ over the $1000$ refinement iterations. We additionally apply Algorithm~\ref{alg:greedy2} (without reinitializing the junction parameters) once every 50 refinement iterations, which we find helps our method avoid getting trapped in local minima. The runtime of our algorithm on an NVIDIA Tesla V100 GPU is $110$ seconds for a $192 \times 192$ image with patch size $R=21$, but both runtime and space usage can be significantly reduced by only considering every $s$th patch in both spatial dimensions for some constant stride $s$ (see supplement for the effect of $s$ on runtime and performance). We implemented our algorithm in PyTorch, and our code and datasets are available on our project page~\cite{projectpage}.

\section{Experiments} \label{section:experiments}

Once an image is analyzed, its field of junctions provides a distributional representation of  boundary structure and smooth regional appearance. Each pixel in the field provides a ``vote'' for a nearby (sub-pixel) vertex location with associated wedge angles and color values around that location.  Simple pixel-wise averages derived from the field are useful for extracting  contours, corners and junctions, and boundary-aware smoothing. We demonstrate these uses here, and we compare our model's regularization to previous methods for curvature minimization.

We evaluate performance using three types of data. First, we show qualitative results on captured photographs. Second, we quantify repeatability using the Smartphone Image Denoising Dataset (SIDD)~\cite{abdelhamed2018high}, evaluating the consistency of extracted boundaries when the same scene is photographed at decreasing light levels (and thus increasing noise levels). Finally, to precisely quantify the accuracy of extracted contours, corners and junctions, we generate a dataset of $300$ synthetic grayscale images (shown in the supplement) with boundary elements known to sub-pixel precision, and with carefully controlled noise levels. In this section we provide results using uncorrelated noise, and our supplement contains results on images corrupted by other noise models.

\boldstart{Boundary-aware smoothing.} A field of junctions readily provides a boundary-aware smoothing using Equation~\ref{equation:colormapsmooth}. An example for a photograph is shown in Figure~\ref{figure:car}, and a comparison of its resilience to noise with that of~\cite{xu2011image} is shown in Figure~\ref{figure:R}.

\boldstart{Boundary Detection.} A field of junctions also immediately provides a boundary map via Equation~\ref{equation:edgemapsmooth}. Figure~\ref{figure:car} shows the resulting boundaries extracted from a photograph, and Figure~\ref{figure:lux} shows a qualitative comparison of our boundaries to previous edge detection and segmentation methods on a patch extracted from a noisy short-exposure photograph.

We quantitatively compare our results to existing contour and boundary detection methods: gPb~\cite{arbelaez2010contour}, HED~\cite{xie2015holistically}, Ofir \etal~\cite{ofir2016fast}, and gPb and HED when denoised by BM3D~\cite{dabov2007image} supplied with true noise level $\sigma$. (Since \cite{ofir2016fast} is designed for low SNR, so we do not combine it with BM3D.) In Figure~\ref{figure:fmeasures} we show the F-scores of results obtained by each method on our synthetic dataset. The F-score is computed by matching the boundaries output by each detector with the ground truth and taking the harmonic mean of its precision and recall.

\begin{figure}[t]
\begin{center}
    \includegraphics[width=\linewidth]{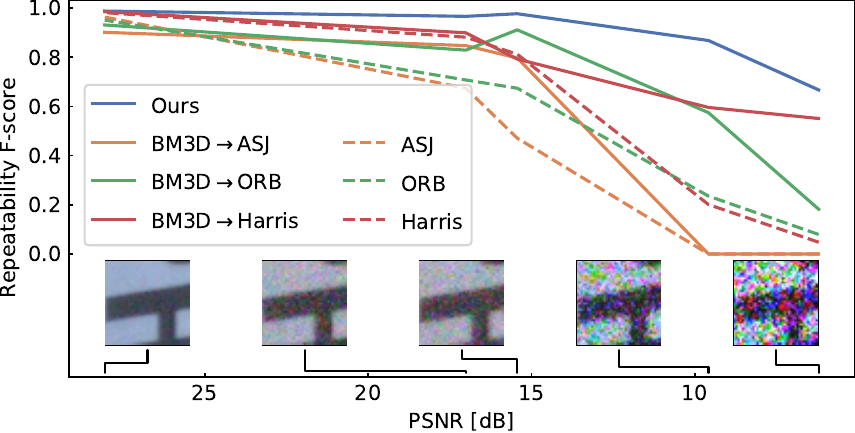}
\end{center}
\vspace*{-5mm}
\caption{Vertex detection repeatability over increasing noise on patches from SIDD, compared to repeatability of other detectors with and without denoising. The number of points detected by each method on the clean ground truth is 126 (ours), 49 (ASJ), 57 (Harris), and 70 (ORB).} \label{figure:junction_repeatability}
\end{figure}

\begin{figure}[t]
\begin{center}
    \includegraphics[width=\linewidth]{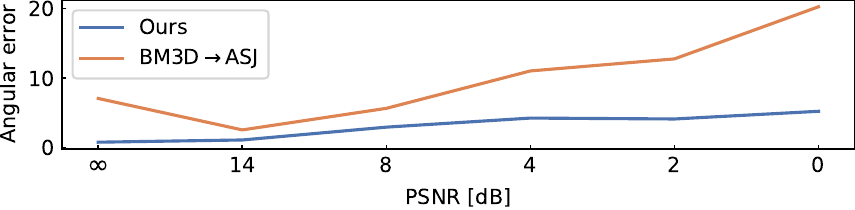}
\end{center}
\vspace*{-5mm}
\caption{Error of angles (in degrees) at detected junctions on our dataset, for our method and ASJ preprocessed by BM3D. Our method degrades slowly. We reported accuracy for ASJ on correctly-detected junctions only.
See Figure~\ref{figure:fmeasures} for sample patches.} \label{figure:junction_angular_error}
\end{figure}

\boldstart{Curvature Regularization.} Figure~\ref{figure:regularization} compares our  boundary regularization to $\ell_1$-elastica~\cite{he2019segmentation} with optimally-tuned parameters. This represents the strongest possible comparison across a large family of existing regularizers, because elastica includes pure-length and pure-curvature minimization as special cases, and because minimizing the $\ell_1$-norm outperforms the $\ell_2$-norm in these images. Unlike existing regularizers, the field of junctions preserves sharp corners; favors linear contours over curved ones; is agnostic to length and convexity of boundaries; and is, as far as we know, the first to do all of this while preserving junctions.

\boldstart{Vertex Detection.} A field of junctions also provides a map of vertex locations that can be used like a traditional corner, junction, or interest point detector. To create a vertex map, we use weighted voting from each junction in the field. The likelihood that a vertex exists at location $\bx$ is:
\begin{equation}
    V(\bx) \propto \sum_{i=1}^N w_i \kappa\left(\bx - \bx_i^{(0)}\right),
\end{equation}
with Gaussian kernel $\kappa(\Delta\bx) = \exp\left(-\frac{\|\Delta\bx\|^2}{2\gamma^2}\right)$ of width $\gamma$, and weights $w_i$ that suppress votes from patches having wedge-angles close to $0^\circ$ or $180^\circ$ (\ie with no unique vertex) and from patches with vertex $\bx^{(0)}$ very far from the patch center. (See supplement for full expression.)

Figure~\ref{figure:R} shows the qualitative results of our vertex detector in the low- and high-noise regime, compared with ASJ~\cite{xue2017anisotropic}. A quantitative study of the robustness of our detector to noise on our synthetic dataset is shown in Figure~\ref{figure:fmeasures}. We again use F-score to compare to ASJ~\cite{xue2017anisotropic} and to BM3D followed by ASJ, and to baseline CNNs that we trained on our dataset specifically for vertex detection. In this experiment, a separate CNN was trained for each PSNR. Figure~\ref{figure:junction_repeatability} shows the repeatability of our vertex detector over different noise levels using patches extracted from SIDD, compared to ASJ~\cite{xue2017anisotropic}, Harris~\cite{harris1988combined}, and ORB~\cite{rublee2011orb}. The repeatability F-scores are computed by comparing the points obtained by each method on the noisy images with its output on the noiseless ground truth images. In all cases we find that our model provides superior resilience to noise. Our detector also provides repeatability over change in viewpoint angle similar to other interest point detectors (see supplement).

In addition to the vertex locations, a field of junctions provides an estimate of the angles of each detected vertex. We treat $\banglesymbol_i$ as an estimate for the angles at a pixel $i$. Figure~\ref{figure:junction_angular_error} shows a comparison of this angle estimation accuracy over multiple noise levels with ASJ preprocessed by BM3D. Because ASJ alone fails at moderate noise levels (see Figure~\ref{figure:fmeasures}), we only plot the results of BM3D followed by ASJ.

\section{Limitations}

The field of junctions is governed by just a few parameters, so compared to deep CNNs it has much less capacity to specialize to non-local patterns of boundary shape and appearance that exist in a particular dataset or imaging modality. Also, as currently designed, it analyzes images at only one scale at a time, with $R$ determining the minimum separation between vertices in the output at that scale. Finally, while the analysis algorithm scales well with image size ($O(N)$, compared to the $O(N^{1.5})$ and $O(N\log N)$ algorithms  of~\cite{ofir2016fast,ofir2019detection}) and has runtime comparable to some other analyzers like gPb, it is slower than feedforward CNNs and dedicated smoothers and contour/corner detectors that are engineered for speed on high-SNR images.

\boldstart{Acknowledgements.} This work is supported by the National Science Foundation under Cooperative Agreement PHY-2019786 (an NSF AI Institute,  \href{http://iaifi.org}{http://iaifi.org}).

{\small
\bibliographystyle{ieee_fullname}
\bibliography{references}
}

\newpage

\end{document}

\title{\normalsize \textbf{Supplemental material accompanying: Dor Verbin and Todd Zickler, ``Field of Junctions: Extracting Boundary Structure at Low SNR", In Proc.~IEEE International Conference on Computer Vision (ICCV), 2021}
\vspace{0.7cm}
\\
\Large Supplemental Material:
\\
\textbf{Field of Junctions: Extracting Boundary Structure at Low SNR}}

%\date{}  % Don't include date
\maketitle

\section{Additional demonstrations of the model's parameters}

Our model has three important adjustable parameters: the patch size $R$, the boundary consistency weight $\lambda_B$, and the color consistency weight $\lambda_C$. In this section we provide more information on the behavior of each of those parameters, as well as examples demonstrating their effect.

The boundary consistency term, governed by $\lambda_B$, promotes agreement on boundaries between overlapping patches. It improves the resilience of the field of junctions to noise, and acts as a curvature regularizer which favors isolated corners and junctions connected by low curvature contours.

The color consistency term, governed by $\lambda_C$, promotes agreement on colors between overlapping patches. It provides a means for long-range color sharing between different patches by penalizing color disagreement at the overlap between different patches.

The patch size $R$ determines the minimal detail scale which the field of junctions can capture. Details that have spatial size smaller than $R$, such as two nearby junctions, cannot be described using a single generalized junction, and are suppressed. In addition, the patch size $R$ is crucial in high-noise usage: it should be large to provide enough information for the initialization algorithm to obtain a good initialization, yet small enough to allow capturing all required structure in the image.

Figure~\ref{figure:fig4} shows an extension of Figure~4 from the main paper. It compares our results with those obtained by the $\ell_1$-elastica and $\ell_2$-elastica Chan-Vese methods~\cite{he2019segmentation}, using various parameters. Both methods use a penalty of the form $\int (a+b|\kappa(s)|^d)ds$, where $d=1$ for $\ell_1$-elastica and $d=2$ for $\ell_2$-elastica, and $\kappa(s)$ denotes the curvature at a point parameterized by $s$. In addition to the two parameters $a$ and $b$, the implementation in~\cite{he2019segmentation} uses additional parameters corresponding to a set of auxiliary variables in an augmented Lagrangian optimization scheme. The $\ell_1$-elastica method introduces four new parameters denoted $r_1$, $r_2$, $r_3$, and $r_4$, and the $\ell_2$-elastica model uses three parameters $r_1$, $r_2$, and $r_3$. These additional parameters also have an effect on the results along with $a$ and $b$, and we adjust them to yield optimal results separately for each image. Figure~\ref{figure:fig4} shows a number of representative results using the two models compared with our field of junctions.

Figure~\ref{figure:regularization} shows the effect of the patch size $R$ in the strong regularization regime, \ie, when $\lambda_B$ is set to a large value. In that case, the only permissible contours are polygons with $\ell_\infty$ side length of at least $R$. 

Figure~\ref{figure:lambdac} shows the effect of the color consistency term in the high noise regime. The color consistency term acts to share color information between overlapping patches, which helps refine the junction parameters at each location. This significantly improves the resilience of the field of junctions to high noise.

\begin{figure}[H]
\begin{center}
\includegraphics[width=\linewidth,height=5cm]{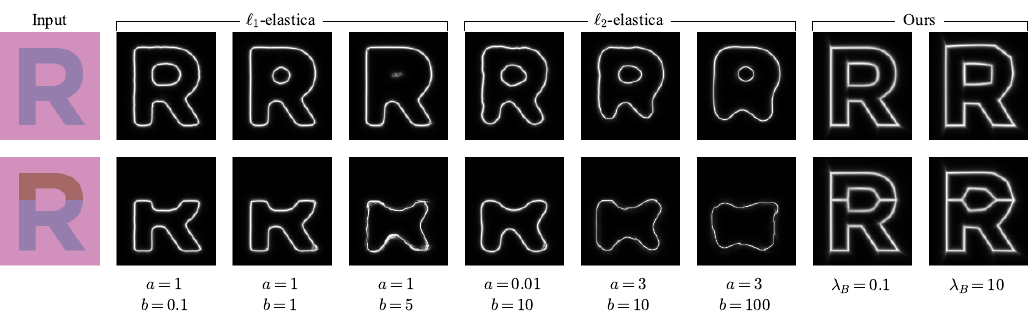}
\end{center}
\caption{An extended version of Figure~4 of our main paper, showing the curvature regularization capabilities of the field of junctions compared to $\ell_1$-elastica and $\ell_2$-elastica with different parameters. The additional parameters of the elastica models were tuned to show a wide range of behaviors. We show the results from the field of junctions using two $\lambda_B$ values, and the other two parameters of the field of junctions were fixed at $R=21$ and $\lambda_C=0$.}
\label{figure:fig4}
\end{figure}

\begin{figure}[H]
\begin{center}
\small
\hspace{1.4cm}
\makebox[0pt]{Input}\hfill
\makebox[0pt]{$R=11$}\hfill
\makebox[0pt]{$R=21$}\hfill
\makebox[0pt]{$R=45$}\hfill
\makebox[0pt]{$R=61$}
\hspace{1.5cm}

\includegraphics[width=\linewidth]{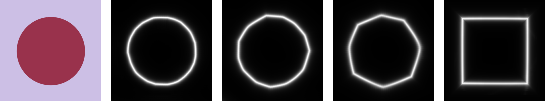}
\end{center}
\caption{The boundary map obtained by the field of junctions with $\lambda_B=100$ and $\lambda_C=0$ for increasing values of $R$. The boundary consistency term penalizes vertices separated by $\ell_\infty$ distance smaller than $R$, which means that only polygons with side length of at least $R$ are permissible when $\lambda_B$ is large.}
\label{figure:regularization}
\end{figure}

\begin{figure}[H]
\begin{center}
\includegraphics[width=\linewidth]{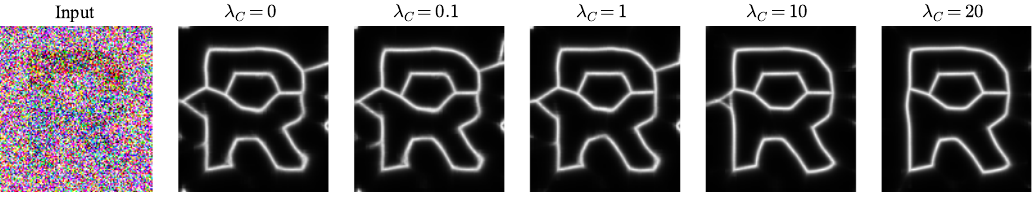}
\end{center}
\caption{The boundary map obtained by the field of junctions with increasing color consistency strength $\lambda_C$. The color consistency term improves the field of junction's resilience to noise by enforcing a global agreement on color values.}
\label{figure:lambdac}
\end{figure}

\newpage

\section{Complete proof of Theorem~1}

We present a complete proof of Theorem~1 from the main paper, showing that Algorithm~1 is guaranteed to find the $M$ true angles of an image (or image patch) $I(\bx)$ containing a single junction, in the noiseless case, and when the vertex position is known. In the main paper we show that in practice Algorithm~2 works very well even when there is substantial noise, or when the vertex position is not known.

For completeness we begin by restating the theorem, omitting the patch index $i$ for convenience:
\begin{theorem}
For a junction image $I(\bx)$ with no noise (i.e., $n\equiv 0$ in Eq.~1 of the main paper) and with vertex $\bx^{(0)}$ known, Algorithm~1 is guaranteed to find the globally optimal angles $\bphi$.
\end{theorem}

We begin by first proving a lemma:
\begin{lemma} \label{lemma:leibniz}
Define $L(\phi) \eqdef \int_{a(\phi)}^{b(\phi)} w(\phi') \|I(\phi') - c^\star(\phi)\|^2 d\phi'$, where $a,b,w\colon[0, 2\pi)\rightarrow\R$ and $I,c^\star\colon [0, 2\pi)\rightarrow \R^K$ are functions such that $c^\star (\phi) \int_{a(\phi)}^{b(\phi)} w(\phi') d\phi' = \int_{a(\phi)}^{b(\phi)} w(\phi') I(\phi')d\phi'$. Then we have:
\begin{equation}
    \frac{d}{d\phi}L(\phi) = w(b(\phi)) \|I(b(\phi)) - c^\star(\phi)\|^2 \frac{d}{d\phi}b(\phi) - w(a(\phi)) \|I(a(\phi)) - c^\star\phi)\|^2 \frac{d}{d\phi}a(\phi).
\end{equation}
\end{lemma}
\begin{proof}
Using the Leibniz integral rule we have the desired equality plus an additional term:
\begin{align}
&\quad \int_{a(\phi)}^{b(\phi)} w(\phi') \frac{\partial}{\partial\phi}\|I(\phi') - c^\star(\phi)\|^2 d\phi' \nonumber \\
&= -2\int_{a(\phi)}^{b(\phi)} w(\phi') \frac{d}{d \phi} c^\star (\phi) \cdot \left[I(\phi') - c^\star (\phi)\right]d\phi' \nonumber \\
&= -2\frac{d}{d \phi} c^\star (\phi) \cdot \int_{a(\phi)}^{b(\phi)} w(\phi') \left[I(\phi') - c^\star (\phi)\right]d\phi' = 0,
\end{align}
where the final equality was obtained using the fact that the integral is zero, due to the construction of $c^\star$ in the lemma.

\end{proof}

\begin{lemma} \label{lemma:sup}
Under the conditions of Theorem~1, the negative log-likelihood function restricted to the $j$th coordinate $\ell_j(\phi) = \ell(\phi^{(1)}, ..., \phi^{(j-1)}, \phi, \phi^{(j+1)}, ..., \phi^{(M)}, x^{(0)}, y^{(0)})$ has no local minima in the open angular interval between any pair of adjacent true angles.
\end{lemma}

\begin{proof}
The junction image $I(\bx)$ is radially symmetric around the known vertex $\bx^{(0)}$, and therefore we can treat it as an angular function $I(\phi)$ of the angle centered at point $\bx^{(0)}$.

From Equation~4 in the main paper, and using polar coordinates relative to the vertex $\bx^{(0)}$ we can write the negative log-likelihood:
\begin{align} \label{equation:loglikelihoodpolar}
    \ell(\phi^{(1)}, ..., \phi^{(M)}, x^{(0)}, y^{(0)}) &= \alpha \sum_{j=1}^M \int_{0}^{2\pi} \int_{0}^{R(\phi')} u^{(j)}(\phi') \|I(\phi') - c^{(j)}\|^2 r' dr' d\phi' \nonumber \\
    &= \alpha \sum_{j=1}^M \int_{0}^{2\pi} \frac{1}{2}R^2(\phi') u^{(j)}(\phi') \|I(\phi') - c^{(j)}\|^2 d\phi' \nonumber\\
    &= \alpha \sum_{j=1}^M \int_{\phi^{(j)}}^{\phi^{(j+1)}} \frac{1}{2}R^2(\phi') \|I(\phi') - c^{(j)}\|^2 d\phi',
\end{align}
where $R(\phi')$ is the distance of each point on the boundary of the $R\times R$ patch at angle $\phi'$ from the vertex $\bx^{(0)}$, and $u^{(j)}(\phi')$ is an angular indicator function returning $1$ if $\phi' \in (\phi^{(j)}, \phi^{(j+1)})$ and $0$ otherwise (with the index $j+1$ computed modulo $M$, and assuming without loss of generality that the angles are non-decreasing, $\phi^{(1)} \leq ... \leq \phi^{(M)}$). The optimal colors $\{c^{(j)}\}$ are given using Equation~3 in the main paper, or in polar coordinates:
\begin{equation} \label{equation:colors}
    c^{(j)} = \frac{\int_{0}^{2\pi}\int_0^{R(\phi')} u^{(j)}(\phi') I(\phi') r' dr'd\phi'}{\int_{0}^{2\pi}\int_0^{R(\phi')} u^{(j)}(\phi')r'dr'd\phi'}
    = \frac{\int_{\phi^{(j)}}^{\phi^{(j+1)}} R^2(\phi') I(\phi')d\phi'}{\int_{\phi^{(j)}}^{\phi^{(j+1)}} R^2(\phi')d\phi'}.
\end{equation}

By applying Lemma~\ref{lemma:leibniz} to Equation~\ref{equation:loglikelihoodpolar} and discarding all terms that do not depend on $\phi^{(j)}$, the partial derivative of the negative log-likelihood function with respect to the $j$th angle in any open interval between two true angles is:
\begin{equation} \label{equation:derivative}
    \frac{d}{d\phi}\ell_j(\phi) = \frac{1}{2}\alpha R^2(\phi) \left[\|I(\phi) - c^{(j-1)}\|^2 - \|I(\phi) - c^{(j)}\|^2\right].
\end{equation}

We finish proving the lemma by contradiction. Let us assume $\phi$ is a local minimum point of $\ell_j$ in the open interval between a pair of true angles. Because $\ell_j$ is smooth within any such interval, we have $\frac{d}{d\phi}\ell_j(\phi) = 0$, and from Equation~\ref{equation:derivative} we have:
\begin{equation} \label{equation:balance}
    \|I(\phi) - c^{(j-1)}\|^2 - \|I(\phi) - c^{(j)}\|^2 = 0.
\end{equation}
Using this expression with the derivative of Equation~\ref{equation:derivative} we obtain that the second derivative at the local minimum is:
\begin{align} \label{equation:secondderivative}
    \frac{d^2}{d^2\phi}\ell_j(\phi)
    &= \alpha R^2(\phi) \left\{\left[\frac{d}{d\phi} I(\phi) - \frac{d}{d\phi} c^{(j-1)}\right] \cdot [I(\phi) - c^{(j-1)}] - \left[\frac{d}{d\phi} I(\phi) - \frac{d}{d\phi} c^{(j)}\right] \cdot [I(\phi) - c^{(j)}]\right\} \nonumber \\
    &= \alpha R^2(\phi) \left\{- \frac{d}{d\phi} {c^{(j-1)}} \cdot [I(\phi) - c^{(j-1)}] + \frac{d}{d\phi} {c^{(j)}} \cdot [I(\phi) - c^{(j)}]\right\},
\end{align}
where we used the fact that $I(\phi)$ is constant for every $\phi$ in an open interval between two true angles.

Multiplying both sides of Equation~\ref{equation:colors} by the denominator of the right hand side for indices $j-1$ and $j$ and differentiating with respect to $\phi$ yields:
\begin{align}
    \frac{d}{d\phi}c^{(j-1)} \int_{\phi^{(j-1)}}^{\phi} R^2(\phi')d\phi' + c^{(j-1)}R^2(\phi) &= R^2(\phi) I(\phi), \nonumber \\
    \frac{d}{d\phi}c^{(j)} \int_\phi^{\phi^{(j)}} R^2(\phi')d\phi' - c^{(j)}R^2(\phi) &= -R^2(\phi) I(\phi).
\end{align}

Rearranging we obtain:
\begin{align}
    R^2(\phi)\left[I(\phi) - c^{(j-1)}\right] &= \frac{d}{d\phi} c^{(j-1)} \int_{\phi^{(j-1)}}^{\phi} R^2(\phi') d\phi', \label{equation:mid1} \\
    R^2(\phi)\left[I(\phi) - c^{(j)}\right] &= -\frac{d}{d\phi} c^{(j)} \int_{\phi}^{\phi^{(j)}} R^2(\phi') d\phi'. \label{equation:mid2}
\end{align}

Finally, we substitute Equations~\ref{equation:mid1} and \ref{equation:mid2} into the expression for the second derivative in Equation~\ref{equation:secondderivative}:
\begin{align}
    \frac{d^2}{d^2\phi}\ell_j(\phi)
    &= \alpha \left\{- \left\|\frac{d}{d\phi} c^{(j-1)} \right\|^2 \int_{\phi^{(j-1)}}^{\phi} R^2(\phi') d\phi'
    - \left\| \frac{d}{d\phi} c^{(j)}\right\|^2 \int_{\phi}^{\phi^{(j)}} R^2(\phi') d\phi'\right\},
\end{align}
and the second derivative at $\phi$ is negative, which is contradictory to our assumption that $\phi$ is a minimum point, thus concluding our proof of the lemma.
\end{proof}

The complete proof of Theorem 1 is then obtained by using Lemma~\ref{lemma:sup} and the proof sketch from the main paper:
\begin{proof}
First, note that $\ell_j(\phi)$ is continuous and smooth for all $\phi$ other than possibly a discontinuity in the derivative at any of the true junction angles. If the optimal $\phi$ is not one of the true junction angles then it must lie in the \emph{open} interval between two such angles, \ie $\phi \in (\phi^-, \phi^+)$. From Lemma~\ref{lemma:sup} we have that $\ell_j(\phi)$ does not have any local minima in $(\phi^-, \phi^+)$, and therefore for each angular interval between two true junction angles the cost function must be minimized at one of the endpoints. Therefore repeatedly minimizing $\ell_j(\phi)$ for $j = 1,..., M$ is guaranteed to provide a globally optimal set of angles $\{\phi^{(1)}, ..., \phi^{(M)}\}$.
\end{proof}

\newpage

\section{Importance of refinement}

Figure~\ref{figure:refinement} shows a comparison of the results obtained by a field of junctions with and without the refinement stage on a photograph. In natural images with low noise, the refinement stage mainly cleans up the boundary maps, only maintaining boundaries corresponding to salient edges and junctions in the image. Because the main effect of the refinement stage in this case is to remove unnecessary boundaries around uniform regions of the image, it does not affect the boundary-aware smoothing, as the pixel colors predicted in both cases are similar.

Figure~\ref{figure:refinement_R} shows the effect of the refinement stage over multiple noise levels. While in the low-noise regime the refinement stage of our algorithm only suppresses superfluous boundaries, in the high-noise regime the spatial consistency enforced by it also significantly improves the boundaries predicted by each junction.

\begin{figure}[H]
 \centering
 \includegraphics[width=0.48\textwidth]{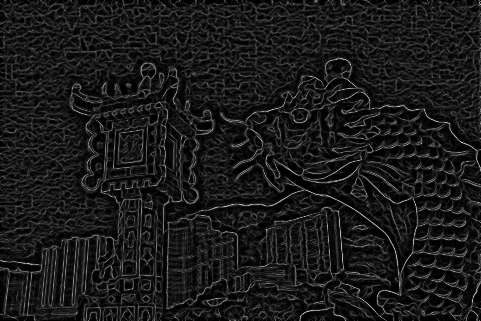} \hfill
 \includegraphics[width=0.48\textwidth]{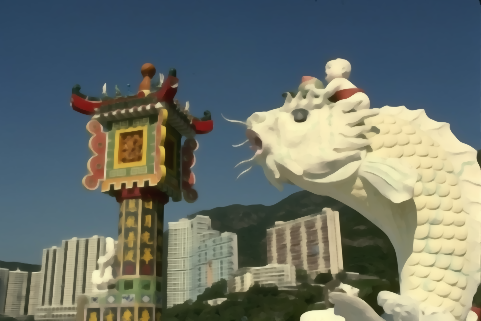} \\
 \includegraphics[width=0.48\textwidth]{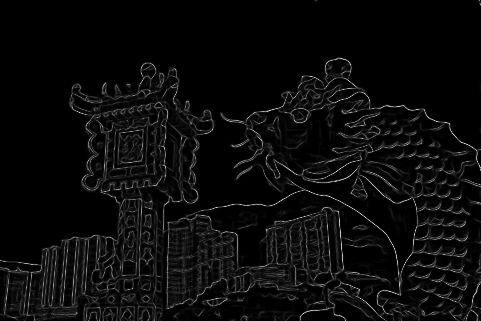} \hfill
 \includegraphics[width=0.48\textwidth]{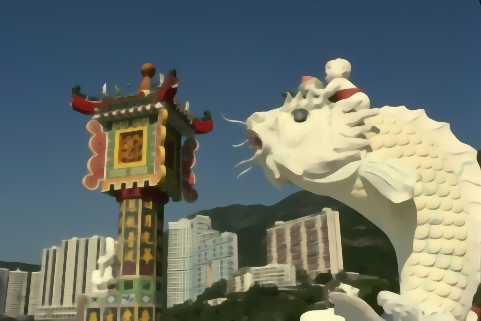} \\
\caption{Comparison of our model without the refinement stage (top) and with it (bottom). The refinement stage improves the boundary map by suppressing boundaries in uniform regions and preventing additional spurious boundaries. In the low-noise regime the refinement stage does not significantly impact the boundary-aware smoothing result, meaning that smoothing using a field of junctions can be accelerated significantly by omitting the refinement stage.}
\label{figure:refinement}
\end{figure}

\begin{figure}[H]
\centering
Noisier $\xrightarrow{\hspace*{15cm}}$\vspace{0.1cm}
\includegraphics[width=\textwidth]{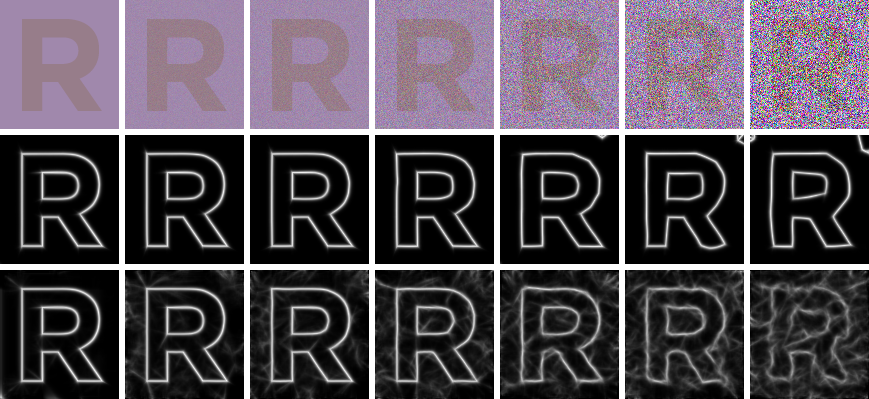}
\caption{Field of junctions obtained with the refinement stage of our algorithm (middle row) and without it (bottom row), on a synthetic image with increasing noise levels. The refinement stage reduces superfluous boundaries across all noise levels, and significantly improves the prediction near true boundaries in the high-noise regime.}
\label{figure:refinement_R}
\end{figure}

\newpage

\section{Parametrization of indicator functions during refinement}

As described in Section~4.2 of the main paper, we replace the indicator functions $\{\mathbf{u}_{\boldsymbol{\theta}}(\mathbf{x})\}$ with smooth functions based on level sets. Figure~\ref{figure:distfunctions} provides an example of this process for $M=3$ with junction angles $\phi^{(1)} = 0$, $\phi^{(2)} = \pi/2$, $\phi^{(3)} = 5\pi/4$, and the junction's vertex is chosen at the center of the patch.

\begin{figure}[H]
\begin{center}
\includegraphics[width=0.55\linewidth]{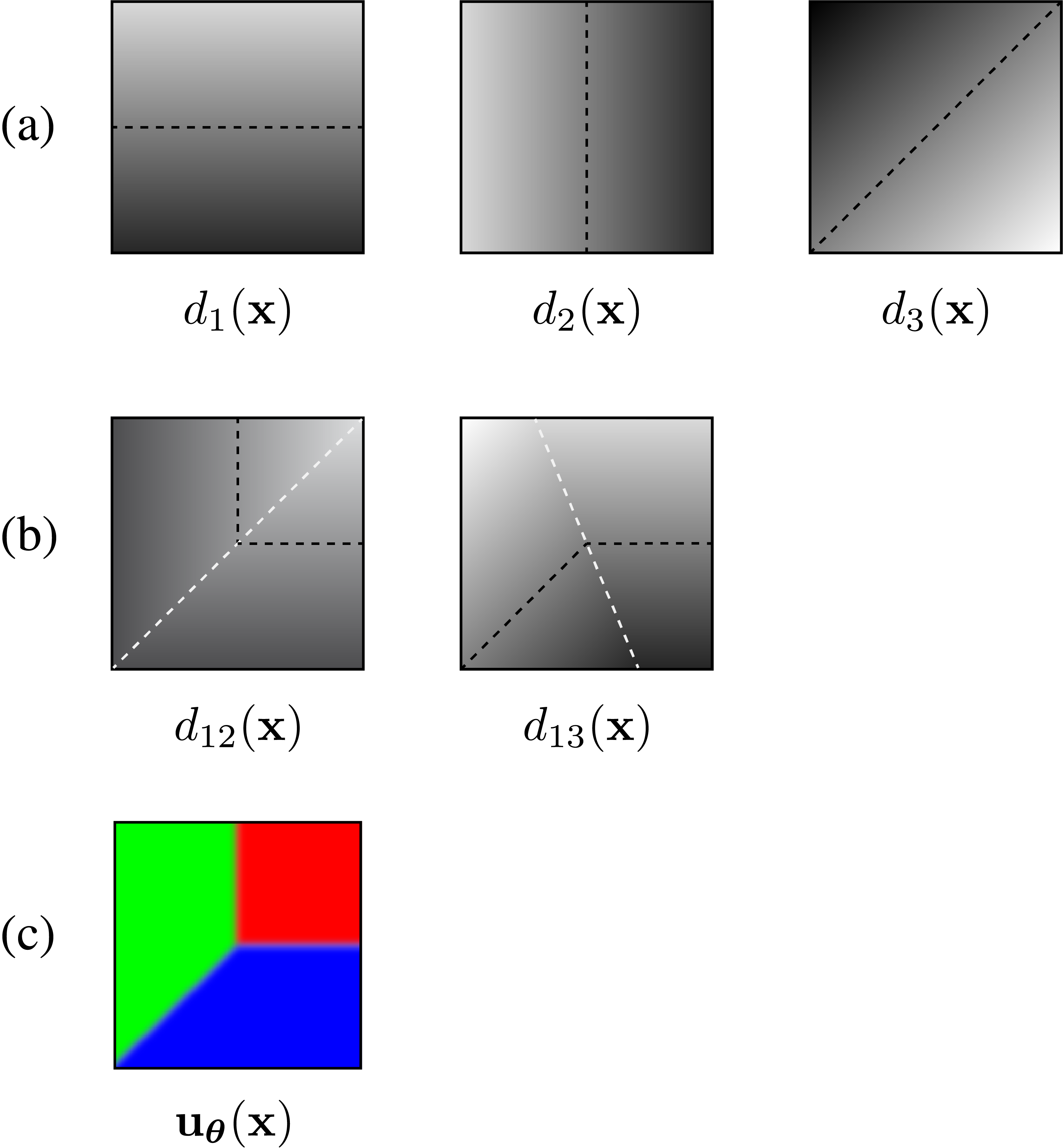}
\end{center}
\caption{Indicator function relaxation based on level-set functions. (a) For the $k$th angle $\phi^{(k)}$ we define a signed distance function from a line passing through the vertex, oriented at angle $\phi^{(k)}$. (b) Each distance function $d_k(\mathbf{x}), k>1$ is combined with the first distance function $d_1(\mathbf{x})$, using Equation~11 from the main paper, to obtain a distance function $d_{1k}(\mathbf{x})$ that is zero-valued at the boundary of the wedge that spans angle $\phi^{(1)}$ through $\phi^{(k)}$. (c) We combine all distance functions $\{d_{1k}(\mathbf{x})\}$ using Equation~12 from the main paper into $M$ smooth indicator functions $\mathbf{u}_{\boldsymbol{\theta}}(\mathbf{x})$, plotted here as a three-channel RGB image, with the blending of the different channels determined by the width of the regularized Heaviside function $\eta$. Black dashed lines mark the zero-level set of their respective distance functions, and white dashed lines in (b) show the set of points $\mathbf{x}$ equidistant from the 1st and $k$th lines, \ie, those satisfying $d_1(\mathbf{x}) = -d_k(\mathbf{x})$.}
\label{figure:distfunctions}
\end{figure}

\newpage

\section{Additional results for Figure~1}

We include additional results corresponding to the photograph in Figure~1 of the main paper. Figure~\ref{figure:fig1} depicts the corners/junctions and boundary-aware smoothing computed by the field of junctions, and a comparison to existing corner/junction detectors (Harris~\cite{harris1988combined}, ASJ~\cite{xue2017anisotropic}) and smoothers (L$_0$~\cite{xu2011image}, Bilateral filter~\cite{tomasi1998bilateral}), all preceded by an optimally-tuned denoiser~\cite{dabov2007image}.

\begin{figure}[H]
\begin{center}
\includegraphics[width=\linewidth]{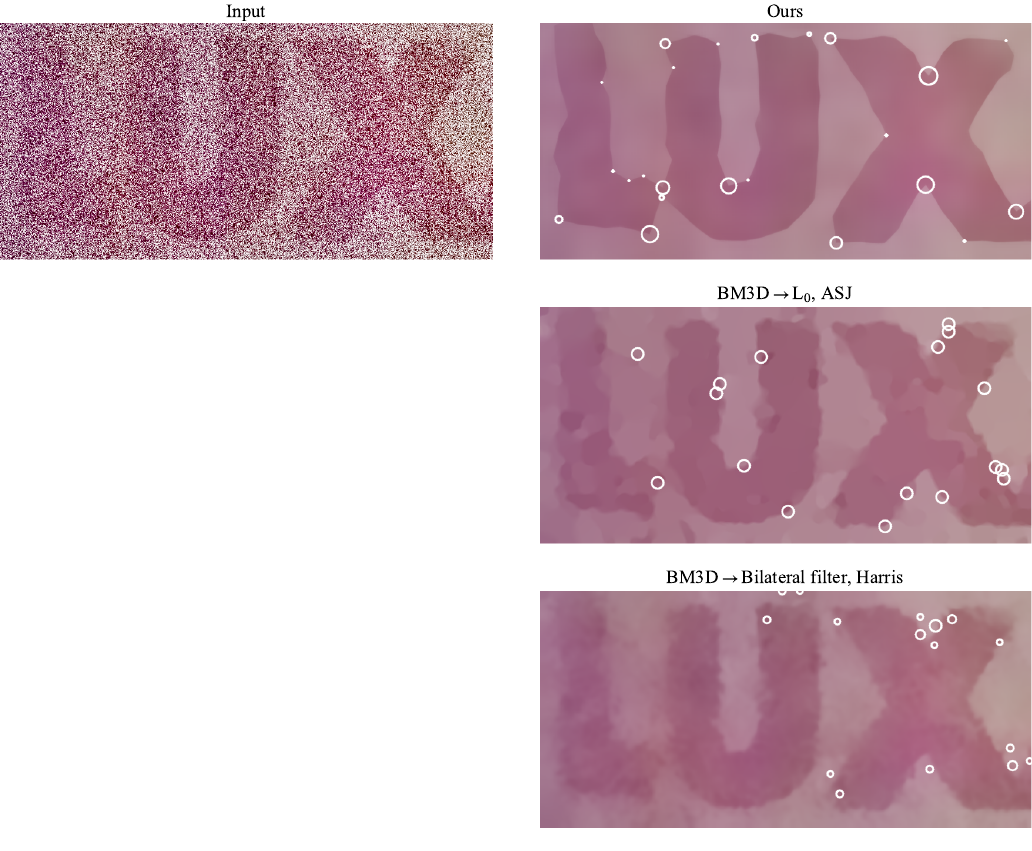}
\end{center}
\caption{Corner and junction detection, and boundary-aware smoothing, for the input image from Figure~1 of the main paper. The smoothed images are shown in color, and the junctions and corners are marked by white circles. For the field of junctions and the Harris corner detector, the radius of each circle is proportional to the confidence of the vertex it denotes.}
\label{figure:fig1}
\end{figure}

\newpage

\section{Additional comparisons with increasing noise}

Figures~\ref{figure:fig2} and~\ref{figure:fig2_boundaries} provide an extended comparison of our method to previous methods, similar to Figure~2 of the main paper. The field of junctions is more stable than previous methods at detecting boundaries, corners and junctions, and smooth colors, even when those are preceded by a denoiser~\cite{dabov2007image} equipped with the true PSNR of the image. While most previous methods are somewhat successful at moderate noise levels when preceded by such a denoiser, when they are run directly on the noisy input image they all fail even at low noise levels (see for example gPb in Figure~\ref{figure:fig2}).

\begin{figure}[H]
\begin{center}
\includegraphics[width=0.94\linewidth]{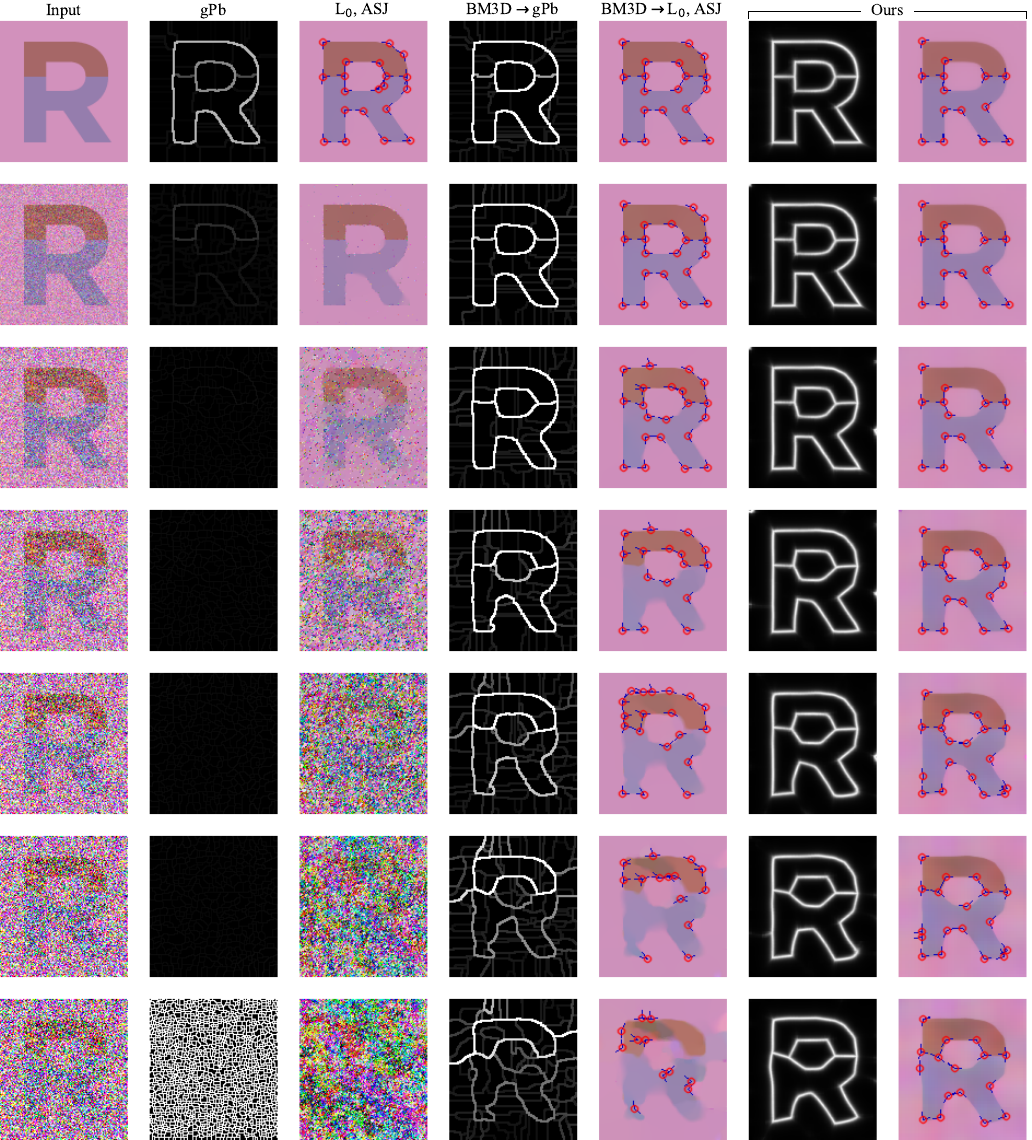}
\end{center}
\caption{An extended version of Figure~2 from the main paper, over a larger range of noise levels. The field of junctions identifies boundaries (column 6), corners/junctions (circles and angles, column 7), and smooth colors (colors, column 7). These are more stable compared to a contour detector (gPb~\cite{arbelaez2010contour}), junction detector (ASJ~\cite{xue2017anisotropic}), and boundary-aware smoother ($\text{L}_0$~\cite{xu2011image}), with and without first applying a denoiser to the input (BM3D~\cite{dabov2007image}).}
\label{figure:fig2}
\end{figure}

\begin{figure}[H]
\begin{center}
\includegraphics[width=0.93\linewidth]{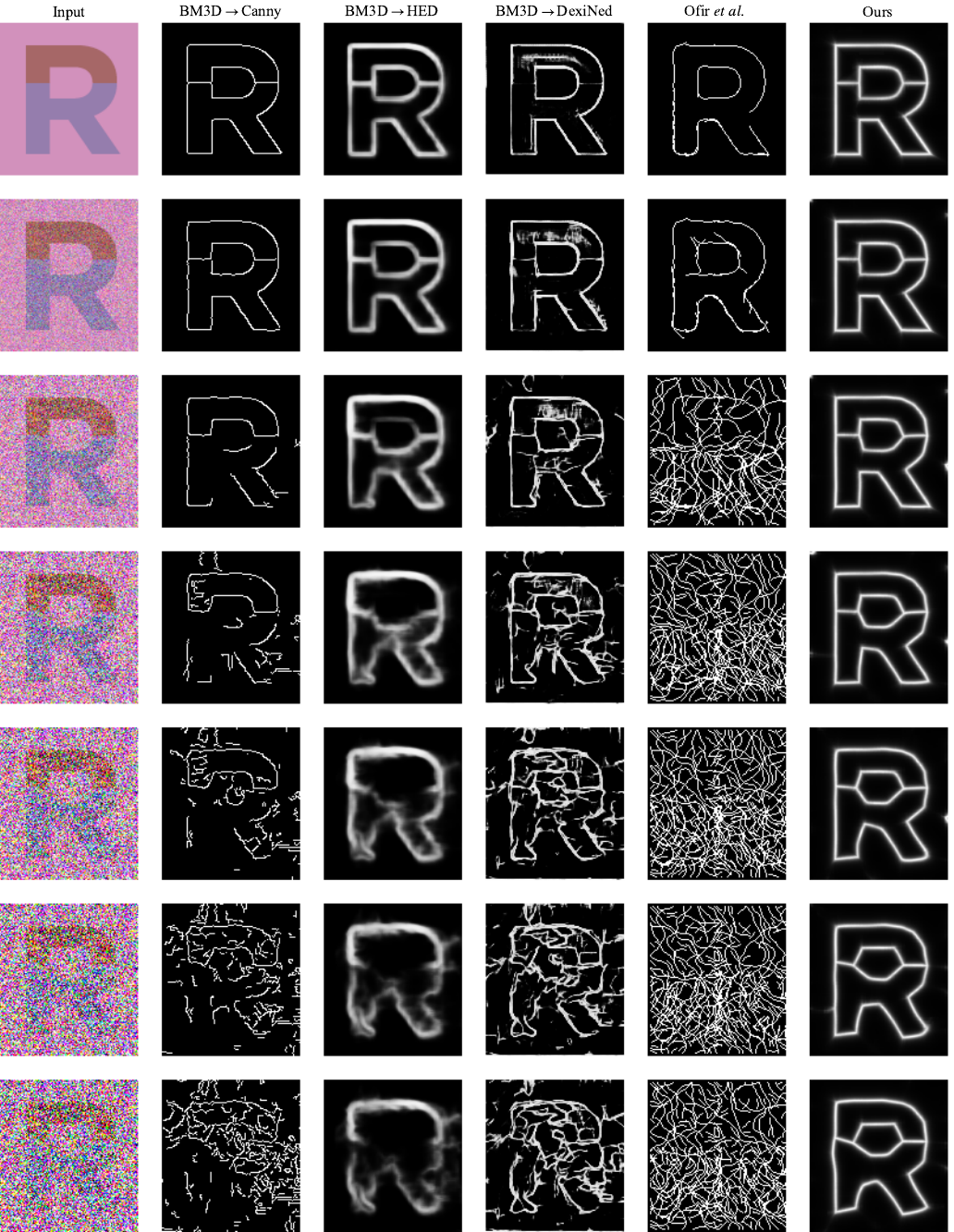}
\end{center}
\caption{Additional comparisons to previous boundary detectors over a large range of noise levels. The methods above include all methods from Figure~1 of the main paper, except for Chan-Vese which does not work when there are more than two image regions.}
\label{figure:fig2_boundaries}
\end{figure}

\newpage

\section{Additional experiments with correlated noise}

Our analysis algorithm is developed using an additive white Gaussian noise model, and in the paper we derive a closed-form solution for the optimal wedge colors under that noise model. In practice we find that the algorithm performs well under other spatially-independent noise models as well, including photographic ISO noise as shown in Figure~1 of the main paper and supplement Figure~\ref{figure:fig1}.

Here we evaluate performance when the noise is spatially correlated. Figure~\ref{figure:perlin_fscores} shows the F-score obtained by our method and previous approaches (similar to the bottom half of Figure~7 in the main paper) on boundary detection in images corrupted by single-scale Perlin noise~\cite{perlin1985image} at different PSNR values and increasingly large spatial noise scales. The figure insets show examples for each noise level and spatial scale.

\begin{figure}[H]
\ContinuedFloat
 \centering
 \includegraphics[width=0.8\textwidth]{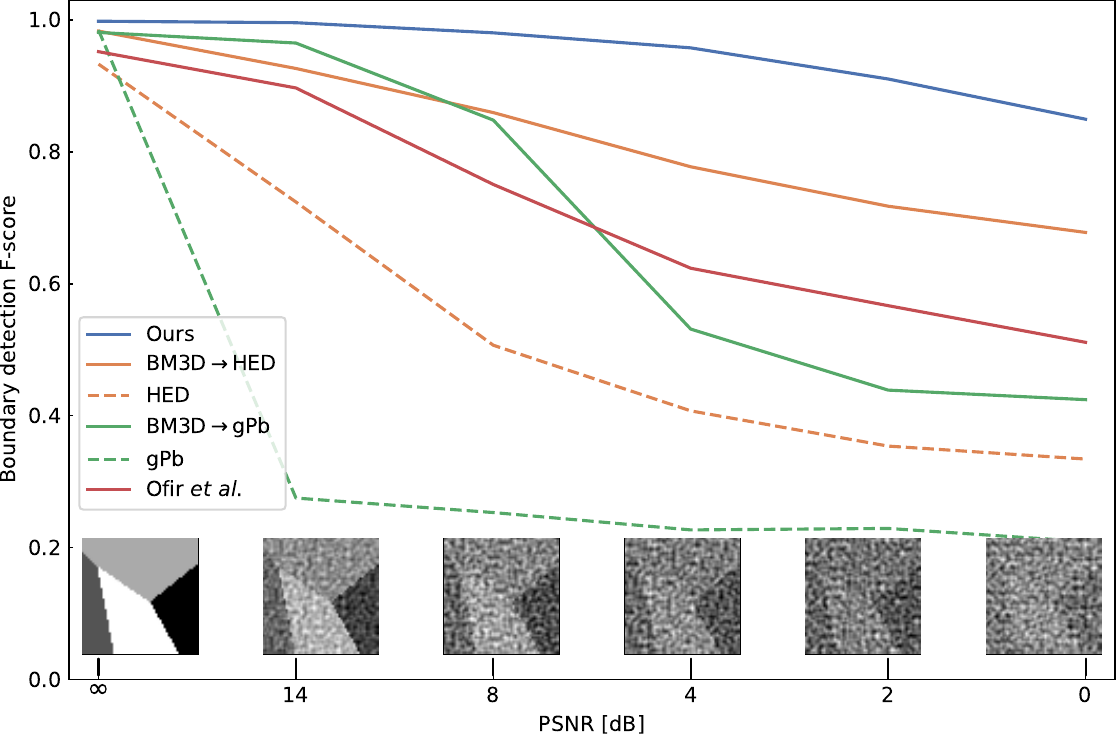}
\caption{}
\end{figure}

\begin{figure}[H]
\ContinuedFloat
 \centering
 \includegraphics[width=0.8\textwidth]{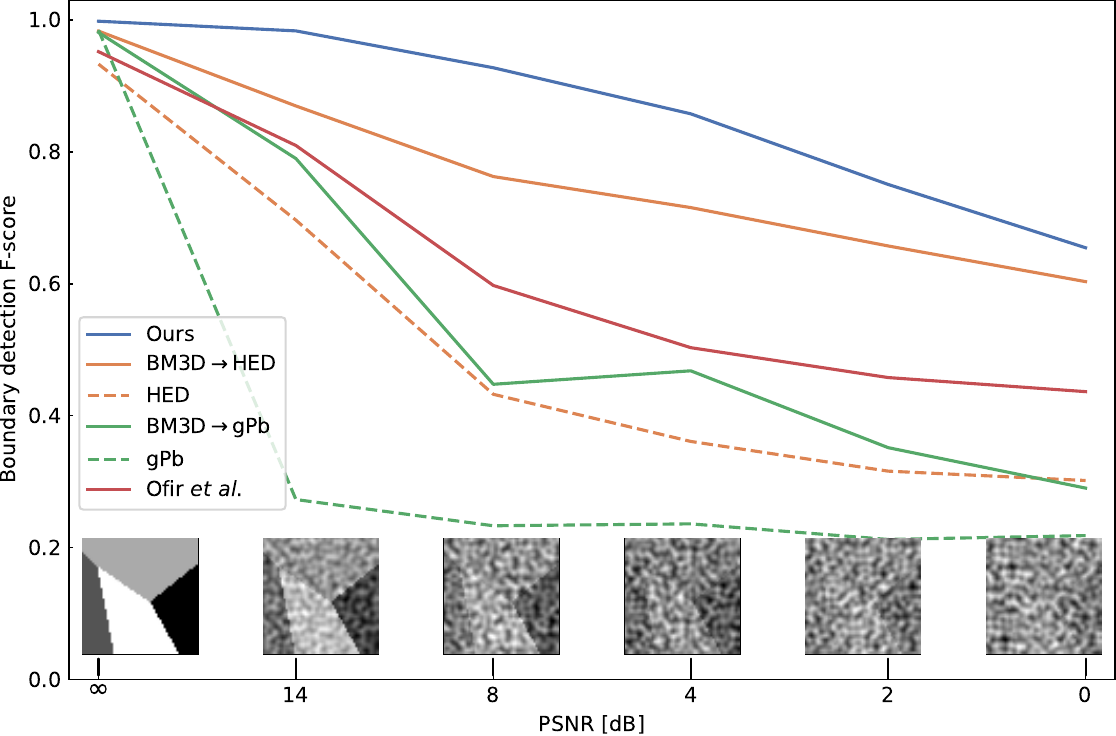}
 \centering
 \includegraphics[width=0.8\textwidth]{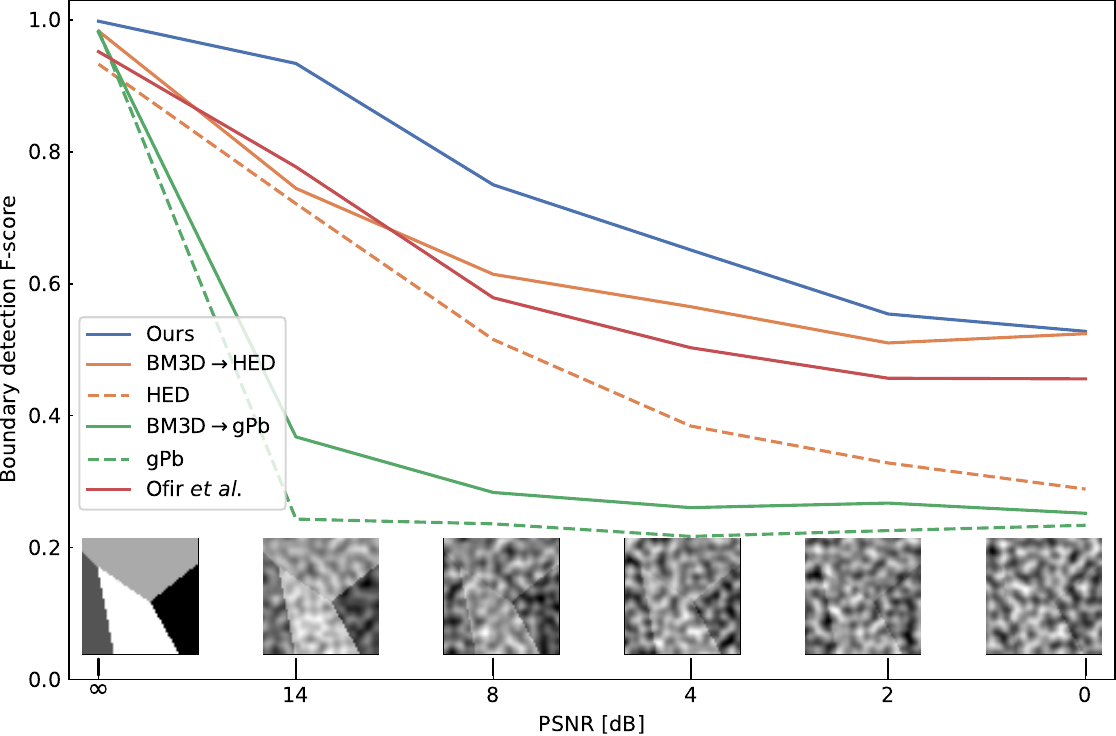}
\caption{\emph{(cont.)} Boundary detection F-score for our dataset corrupted by increasing levels of Perlin noise, at small spatial scale (top), intermediate spatial scale (center), and large spatial scale (bottom). Our model is comparable to previous edge detectors at low noise but outperforms them significantly at high noise, even when the previous methods are preceded by BM3D denoising. At very high noise levels and large noise scale (bottom), all methods obtain a relatively low F-score, and estimating the underlying boundary structure from the corresponding noisy image patch is extremely difficult even for a human observer. Insets: sample patch at different noise levels and spatial noise scales.} \label{figure:perlin_fscores}
\label{figure:perlin_fscores10}
\end{figure}

\newpage

\section{Runtime and effect of stride}

Figure~\ref{figure:stride} shows the effect of using a strided field of junctions for various stride factors $s$. The runtime of our method as a function of $s$ for a fixed image size and patch size $R$, is shown in Figure~\ref{figure:stride_runtime}. Setting $s>1$ significantly accelerates optimization, and does not significantly affect the results, especially at high or intermediate SNRs. Finally, Figure~\ref{figure:N_runtime} shows runtime as a function of image size $N$. Because our algorithm computes all junction updates in parallel, runtime grows very slowly as the image size $N$ is increased.

\begin{figure}[H]
\includegraphics[width=\linewidth]{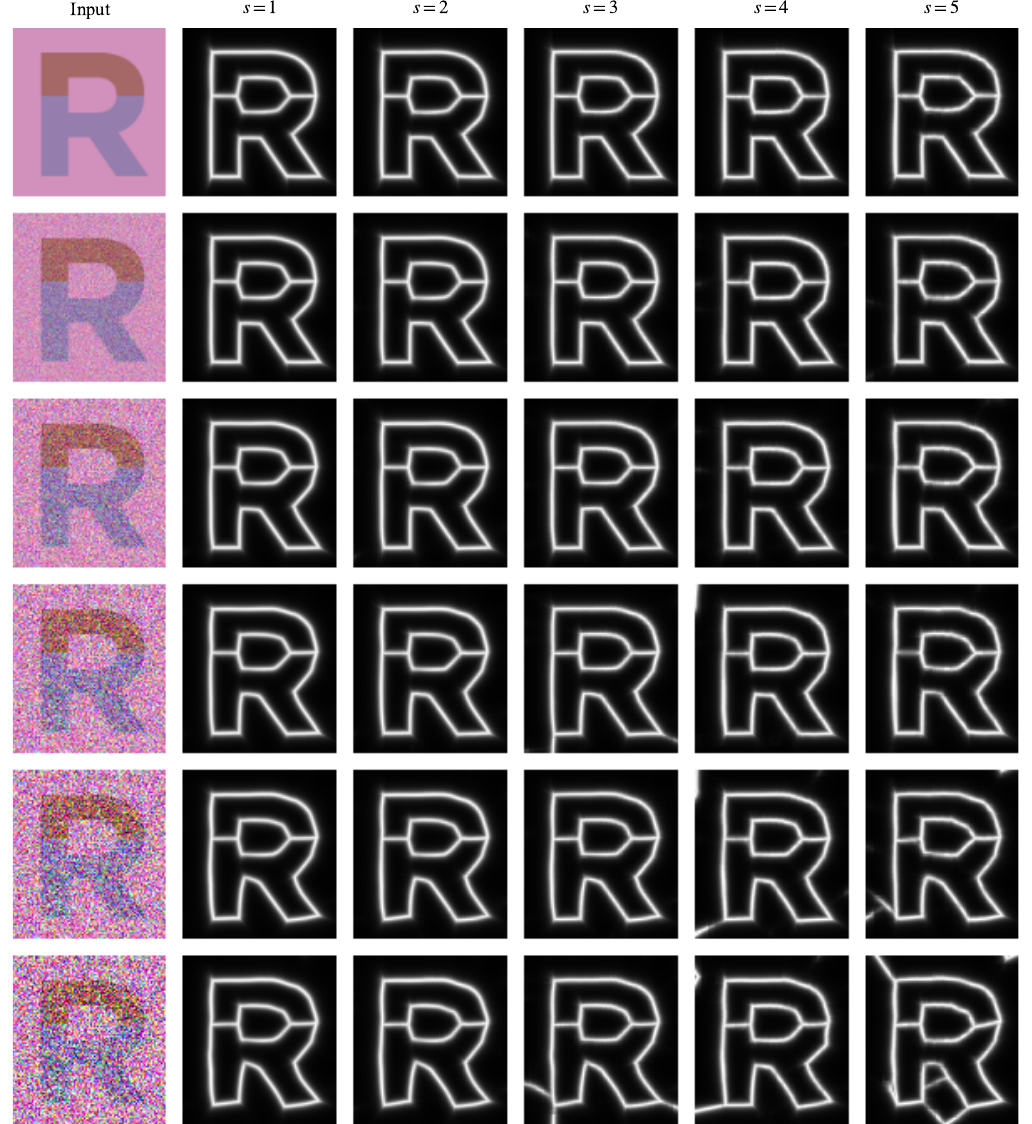}
\caption{\label{figure:stride} Boundary map obtained from the field of junctions for increasing  noise levels, and various strides $s$. At high noise levels the boundaries slowly degrade as $s$ is increased, with a significant runtime improvement (see Figure~\ref{figure:stride_runtime}).}
\end{figure}

\begin{figure}[H]
\begin{center}
\includegraphics[width=0.6\linewidth]{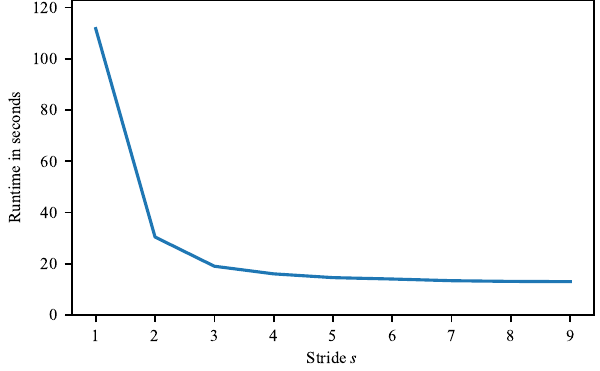}
\end{center}
\caption{\label{figure:stride_runtime} Runtime (in seconds) for different stride values $s$. Runtimes were computed using an NVIDIA Tesla V100 GPU on an image of size $192\times 192$ with patch size $R=21$.}
\end{figure}

\begin{figure}[H]
\begin{center}
\includegraphics[width=0.6\linewidth]{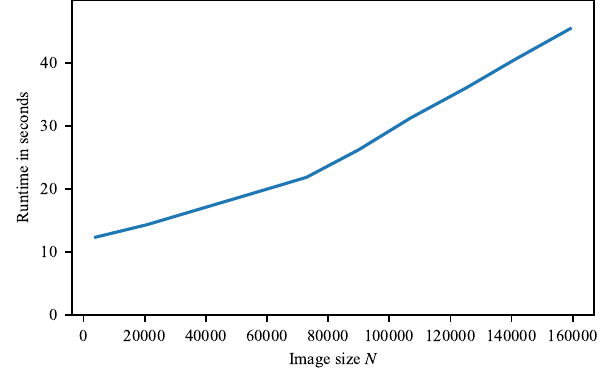}
\end{center}
\caption{\label{figure:N_runtime} Runtime (in seconds) for different image sizes $N$. Runtimes were computed using an NVIDIA Tesla V100 GPU using stride $s=3$ and patch size $R=21$.}
\end{figure}

\newpage

\section{Complete specification of vertex detector}

The vertex probability for each pixel $V(\bx)$ is described in Equation~13 of the main paper. We repeat the expression here with a full explanation of all of its terms. We set:
\begin{align}
    V(\bx) &= \sum_{i=1}^N w_i \kappa\left(\bx - \bx_i^{(0)}\right), \\
    \kappa(\Delta\bx) &= \exp\left(-\frac{\|\Delta\bx\|^2}{2\gamma^2}\right), \\
    w_i &= \exp\left(-\frac{\|\bx_i^{(0)} - \bp_i\|^2}{2\nu_d^2}\right)\cdot \max_{k\neq l} \min\left\{1, \left[1 + \cos\left(\phi^{(k)} - \phi^{(l)}\right)\right]\left[1 - \left|\cos\left(\phi^{(k)} - \phi^{(l)}\right)\right|\right]^{\nu_e}\right\},
\end{align}
where $\bp_i$ is the position of the center of the $i$th patch and $\gamma$, $\nu_d$, $\nu_e$ are parameters.

Each weight $w_i$ is composed of two terms, such that a junction only contributes to the detector if its vertex position is sufficiently close to the center of the patch, and at least a pair of its wedges has an angle not too close to $0$ or $\pi$. The first term is simply a Gaussian term in the distance from the vertex to the patch center, and the second is a function that returns $1$ if any of the junction's wedges has an angle close to $\pi/2$ and drops to $0$ if all angles are $0$ or $\pi$. The second term is chosen to be asymmetric around $\pi/2$ since angles larger than $\pi/2$ are more likely to be caused by structure in the image with nonzero curvature than those below $\pi/2$, and thus should have a lower weight. In all experiments in our paper we use $\gamma = 2R$, $\nu_d = R/2$, $\nu_e = 2$.

\newpage

\section{Synthetic dataset}

Figure~\ref{figure:dataset} shows representative images selected from our synthetic dataset. Our dataset is designed to provide control over noise levels and allow access to the ground truth edges and vertices, so that accurate quantitative analysis can be made. Our dataset contains patches that our field of junctions can handle exactly (uniformly-colored regions, edges, corners, and junctions), as well as ones that it can only model approximately (curved edges, curved junctions, \etc).

Our dataset contains $300$ grayscale images, $100$ of each of three types:
\begin{enumerate}
    \item Images containing two curved objects of gray levels $128$ and $255$ on a black background (see Figure~\ref{figure:dataseta}).
    \item Images containing two rotated squares of gray levels $128$ and $255$ on a black background (see Figure~\ref{figure:datasetb}).
    \item Images containing $4$ regions of gray levels $0$, $85$, $170$, and $255$ separated by two $3$-junctions (see Figure~\ref{figure:datasetc}).
\end{enumerate}

\begin{figure}[H]
     \centering
     \begin{subfigure}[b]{0.3\textwidth}
         \centering
         \includegraphics[width=\textwidth]{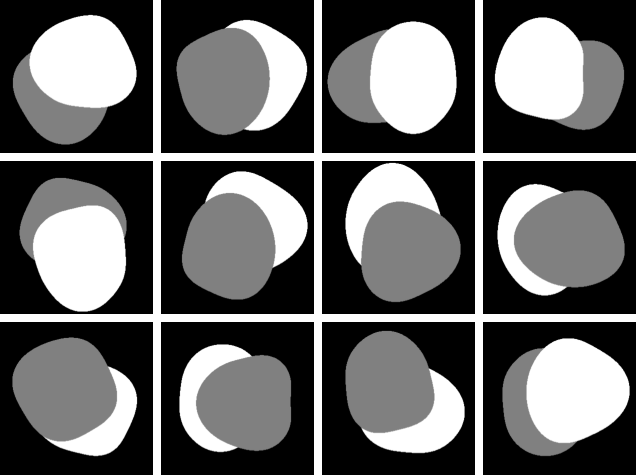}
         \caption{}
         \label{figure:dataseta}
     \end{subfigure}
     \hfill
     \begin{subfigure}[b]{0.3\textwidth}
         \centering
         \includegraphics[width=\textwidth]{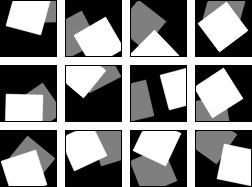}
         \caption{}
         \label{figure:datasetb}
     \end{subfigure}
     \hfill
     \begin{subfigure}[b]{0.3\textwidth}
         \centering
         \includegraphics[width=\textwidth]{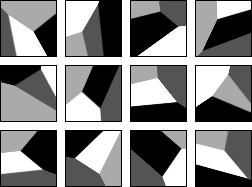}
         \caption{}
         \label{figure:datasetc}
     \end{subfigure}
        \caption{Images from our synthetic dataset.}   
         \label{figure:dataset}
\end{figure}

\newpage

\section{Comparison with the $4$-junction model}

In our main paper all experiments are done using the $3$-junction model (\ie we set $M=3$). However, using a $4$-junction model ($M=4$) can provide an advantage in certain applications which tend to contain X-junctions. Some examples for such cases are images containing specific textures like checkerboards, semi-transparent objects, or reflections from transparent ones~\cite{metelli1974perception}. Figure~\ref{figure:3v4checkerboard} demonstrates how our model behaves in one such example taken from SIDD~\cite{abdelhamed2018high}.

Figure~\ref{figure:3v4} shows a comparison of a field of junctions of a photograph analyzed using our $M=3$ model with that obtained by our $M=4$ model. The differences between the two are usually insignificant in photographs, since those typically do not contain many salient X-junctions.

Figures~\ref{figure:3v4_increasing_noise_R} and~\ref{figure:3v4_increasing_noise} show a comparison of the two models with increasing noise levels, on synthetic and real images. In images not containing X-junctions both models perform similarly, with the $3$-junction model slightly outperforming the $4$-junction model at very high noise levels. Because the $4$-junction model contains an additional angle it tends to either output more superfluous boundaries than the $3$-junction model (as in Figure~\ref{figure:3v4_increasing_noise}), or in some cases less of them (as in Figure~\ref{figure:3v4_increasing_noise_R}).

\begin{figure}[H]
 \centering
 \includegraphics[width=0.8\textwidth]{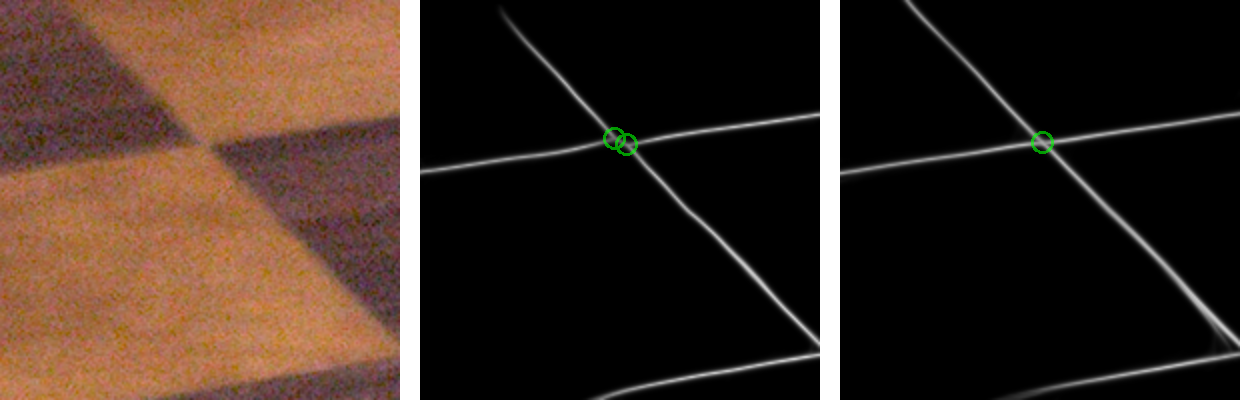}
\caption{Close-up of the output of our model with $M=3$ (center) and with $M=4$ (right) on a photograph from SIDD. Here, the latter allows for better localization of the X-junction and the boundaries in its neighborhood}
\label{figure:3v4checkerboard}
\end{figure}

\begin{figure}[H]
 \centering
 \includegraphics[width=0.48\textwidth]{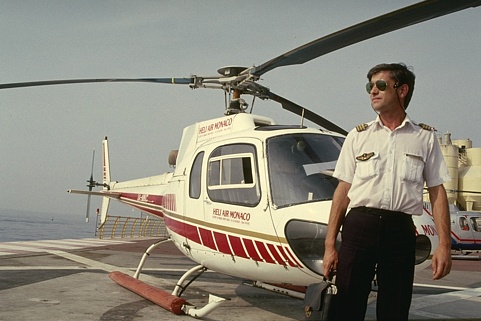} \hfill
 \includegraphics[width=0.48\textwidth]{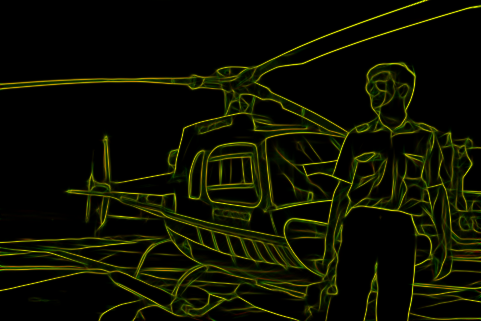}
\caption{Boundary maps obtained by our field of junctions using the $3$-junction and $4$-junction models. The boundary maps obtained from the input image (left) using $M=3$ and $M=4$ are plotted (right) in the red and green channels respectively. Salient boundaries generally appear yellow, corresponding to a perfect agreement between the two models. Some boundaries appear green, corresponding to ones predicted only by the $4$-junction model, which tends to suppress small-scale details less than the $3$-junction one.}
\label{figure:3v4}
\end{figure}

\begin{figure}[H]
\centering
\includegraphics[width=0.8\textwidth]{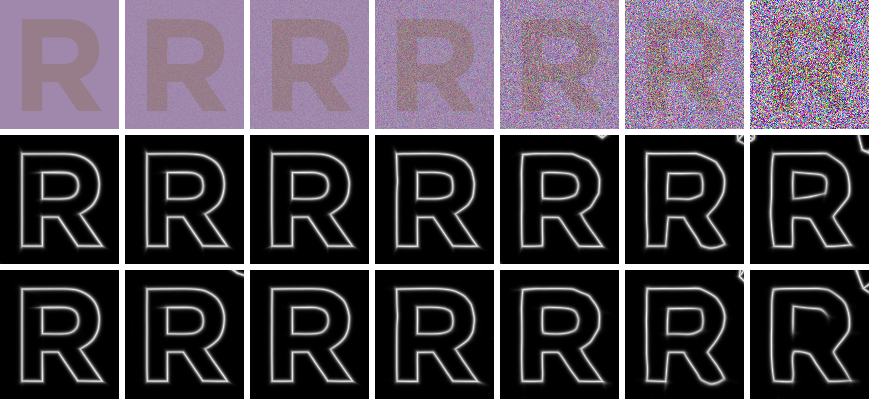}
\caption{Boundary maps obtained using the $3$-junction and $4$-junction models on a synthetic image under increasing noise. At low noise both models produce similar results, but at very high noise, the $3$-junction model performs slightly better.}
\label{figure:3v4_increasing_noise_R}
\end{figure}

\begin{figure}[H]
\centering
\includegraphics[width=0.8\textwidth]{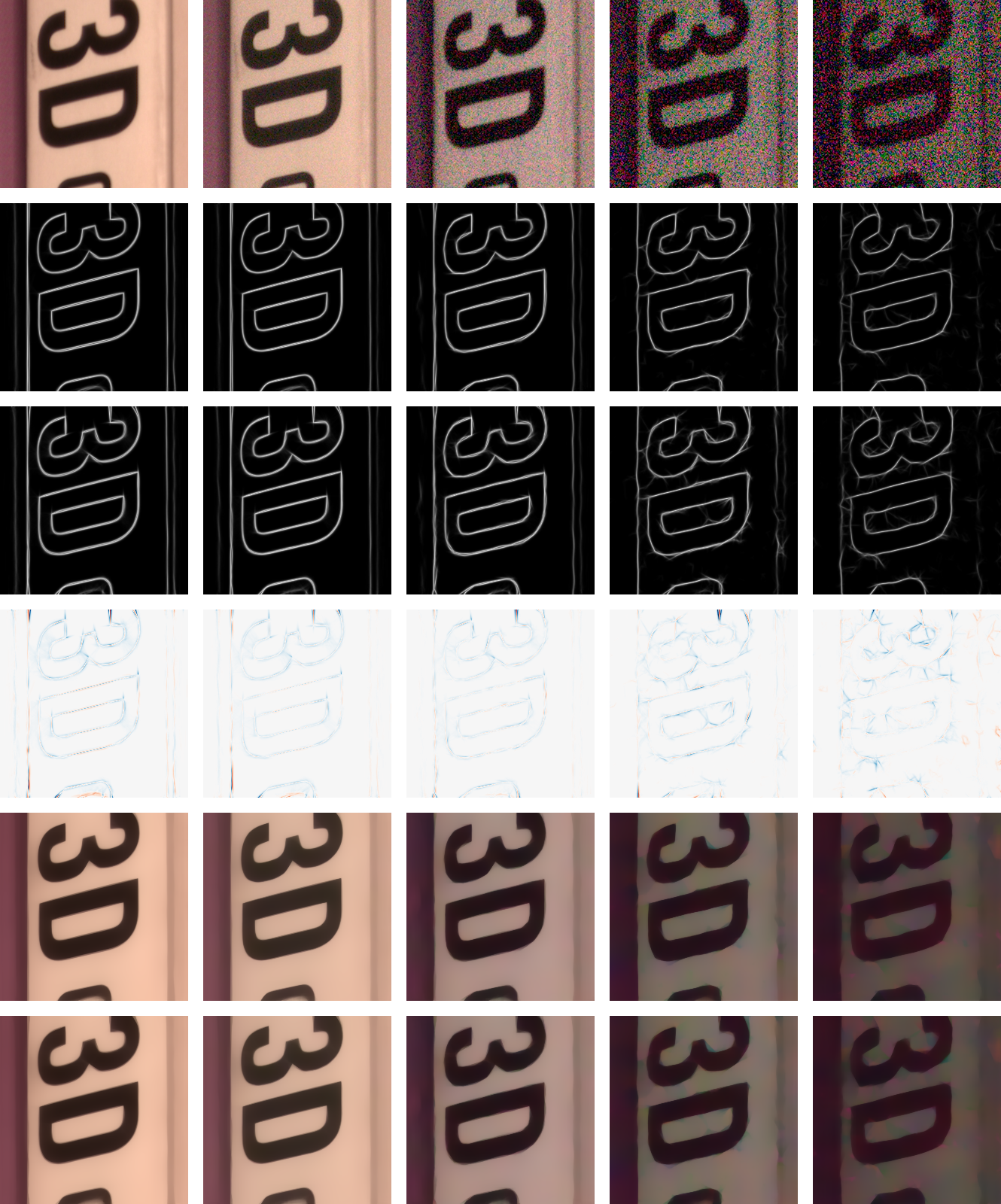}
\caption{Boundary maps obtained using the $3$-junction and $4$-junction models on images from SIDD. The boundary maps output by the $4$-junction model (2nd row) are similar to those output by the $3$-junction model (3rd row), but typically contain more boundaries. The difference between the two boundary maps is shown in the 4th row, where blue pixels are those predicted as boundaries by the $4$-junction model and not by the $3$-junction model, and red pixels correspond to the opposite case. The boundary-aware smoothing results (using $M=3$ in 5th row, $M=4$ in 6th row) are not significantly affected by the number of angles.} %\\\todo{Fix contrast.}}
\label{figure:3v4_increasing_noise}
\end{figure}

\newpage

\section{Comparison with the linear color model}

In our main paper all experiments are done using the constant color model. However, in certain situations higher-order models are beneficial. In fact, any color model where the color functions are linear combinations of functions of the coordinates has a closed-form solution for the color functions. One simple example is a linear color model, where the color in each wedge is a linear function of the coordinates $\bx = (x, y)$. The linear color model outperforms the constant color model in specific applications. We first present a closed-form solution for the color functions in the linear color model, and we then show an example where the linear color model outperforms the constant one.

\subsection{Derivation}

We begin by deriving a closed-form solution for the optimal color functions under the linear model, corresponding to Equation~10 of the main paper.

Similar to Equation~9 in the main paper, we can find the optimal color functions separately for each patch $i$ and wedge $j$. The optimal color function $\jparamc_i^{(j)}$ at the $j$th wedge of the $i$th patch is the solution to:
\begin{align} \label{program:MAPlinear}
    \underset{\jparamc_i^{(j)}}{\min} \int u^{(j)}_{\bjparams_i}(\bx) \left\|\jparamc_i^{(j)}(\bx) - I_i(\bx) \right\|^2 d\bx  +
    \lambda_C \int u_{\bjparams_i}^{(j)}(\bx) \left\|\jparamc_i^{(j)}(\bx) - \hat{I}_i(\bx)\right\|^2 d\bx,
\end{align}
where the boundary consistency term was dropped, as it does not depend on the color functions $\bbjparamc$.

Assuming a linear color model means we can write $\jparamc_i^{(j)}(\bx) = a_i^{(j)} x + b_i^{(j)} y + d_i^{(j)}$, where $a_i^{(j)}, b_i^{(j)}, d_i^{(j)}\in\R^K$, for all $\bx = (x, y)$. Plugging into Equation~\ref{program:MAPlinear}, we get:
\begin{align} \label{program:MAPlinear2}
    \underset{a_i^{(j)}, b_i^{(j)}, d_i^{(j)}}{\min} \int u^{(j)}_{\bjparams_i}(\bx) \left\|a_i^{(j)} x + b_i^{(j)} y + d_i^{(j)} - I_i(\bx) \right\|^2 d\bx  +
    \lambda_C \int u_{\bjparams_i}^{(j)}(\bx) \left\|a_i^{(j)} x + b_i^{(j)} y + d_i^{(j)} - \hat{I}_i(\bx)\right\|^2 d\bx.
\end{align}

The objective in Equation~\ref{program:MAPlinear2} is quadratic in $a_i^{(j)}$, $b_i^{(j)}$, and $d_i^{(j)}$, and therefore it suffices to set its gradient with respect to the three coefficients to zero, yielding three equations linear in  $a_i^{(j)}, b_i^{(j)}, d_i^{(j)}$:

\begin{align} \label{equation:gradientequations}
    \iint u_{\bjparams_i}^{(j)}(\bx) \left(a_i^{(j)} x + b_i^{(j)} y + d_i^{(j)} - [I(\bx) + \lambda_C \hat{I}_i(\bx)]\right)x dxdy = 0, \nonumber \\
    \iint u_{\bjparams_i}^{(j)}(\bx) \left(a_i^{(j)} x + b_i^{(j)} y + d_i^{(j)} - [I(\bx) + \lambda_C \hat{I}_i(\bx)]\right)y dxdy = 0, \\
    \iint u_{\bjparams_i}^{(j)}(\bx) \left(a_i^{(j)} x + b_i^{(j)} y + d_i^{(j)} - [I(\bx) + \lambda_C \hat{I}_i(\bx)]\right) dxdy = 0. \nonumber
\end{align}

The set of three linear equations in~\ref{equation:gradientequations} can be solved by separately inverting a $3\times 3$ matrix and multiplying it by a $3$-vector for each of the $K$ components of $a_i^{(j)}$, $b_i^{(j)}$, and $d_i^{(j)}$.

\newpage

\subsection{Results}

An example for a natural application for the linear color model is when the images contain large regions in which the colors varies linearly with position. Figure~\ref{figure:linearcolormodel} shows an example for this type of image. Our linear color model perfectly explains the image, while the constant color model provides a number of additional low-confidence boundaries throughout the image. 

\begin{figure}[H]
\begin{center}
\includegraphics[width=0.7\linewidth]{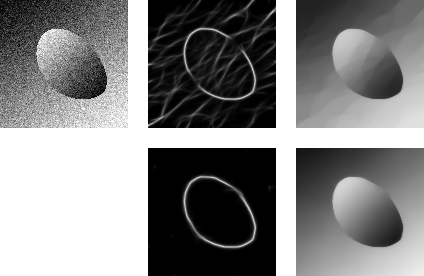}
\end{center}
\caption{Field of junctions for an image containing linear color regions (left) using the constant color model (top row) compared with the linear color model (bottom row). The linear color model provides a better result in this type of application, but we found the constant color model to be sufficient for the experiments in our paper, and specifically for photographs. Input image taken from~\cite{mukherjee2014region}.}
\label{figure:linearcolormodel}
\end{figure}

\newpage

\section{Application to RGB-D images}

In our main paper all experiments are done using grayscale ($K=1$) or RGB images ($K=3$). We find that our field of junctions is also useful for RGB-D images ($K=4$), allowing recovery of missing depth values from datasets such as the Middlebury Stereo Dataset~\cite{scharstein2014high}. Figure~\ref{figure:rgbd} shows the depth values obtained by first analyzing an image into its field of junctions by ignoring pixels with missing depth values, and using the $4$-channel color functions $\{c_i^{(j)}\}$ to compute the global color maps of the input RGB-D image using Equation~8 in the main paper.

\begin{figure}[H]
\begin{center}
\includegraphics[width=0.7\linewidth]{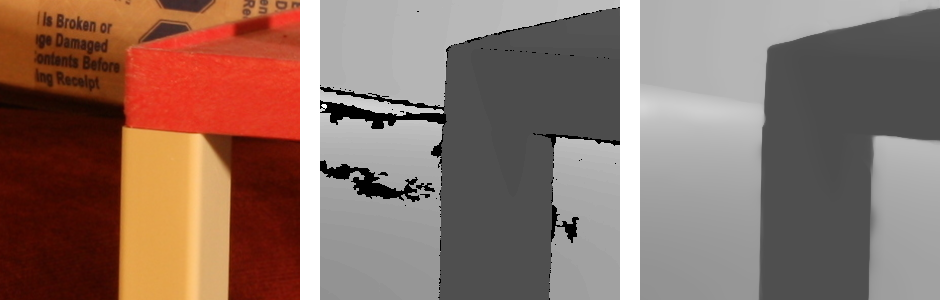}
\end{center}
\caption{Analyzing an image into its field of junctions can be used for filling missing depth values in stereo datasets such as Middlebury. The input RGB image (left) and depth map with missing values (center) are analyzed into their 4-channel field of junctions. The depth channel of the resulting boundary-aware smoothed RGB-D image (right) smoothly fills the missing values.}
\label{figure:rgbd}
\end{figure}

\newpage

\section{Repeatability over change in viewpoint}

We test the repeatability of our vertex detector with respect to change in viewpoint using the Graffiti dataset of~\cite{mikolajczyk2005comparison}, which includes images of a planar surface from angles of 20 degrees to 60 degrees. Figure~\ref{figure:viewpoint} shows the repeatability of our vertex detector compared to that of other interest point detectors, as well as the output of our vertex detector on three of the images in the dataset. 

\begin{figure}[H]
     \centering
     \begin{subfigure}[b]{0.48\textwidth}
         \centering
         \includegraphics[width=\textwidth]{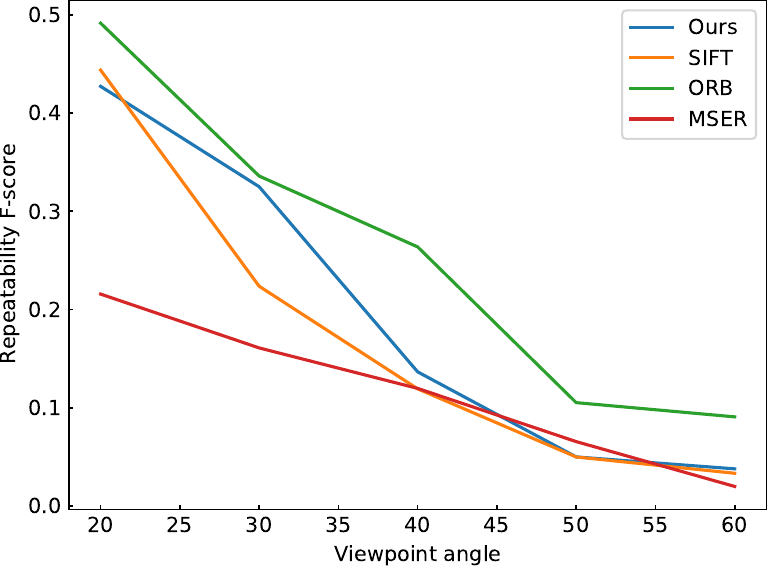}
         \caption{}
         \label{figure:viewpointplot}
     \end{subfigure}
     \hfill
     \begin{subfigure}[b]{0.48\textwidth}
         \centering
         \includegraphics[width=\textwidth]{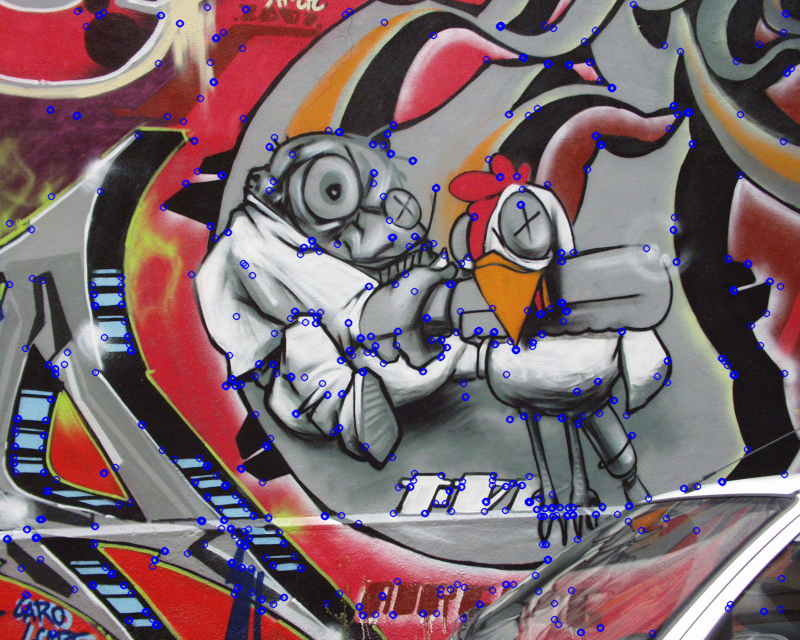}
         \caption{}
         \label{figure:viewpointa}
     \end{subfigure}
     \hfill \\
     \begin{subfigure}[b]{0.48\textwidth}
         \centering
         \includegraphics[width=\textwidth]{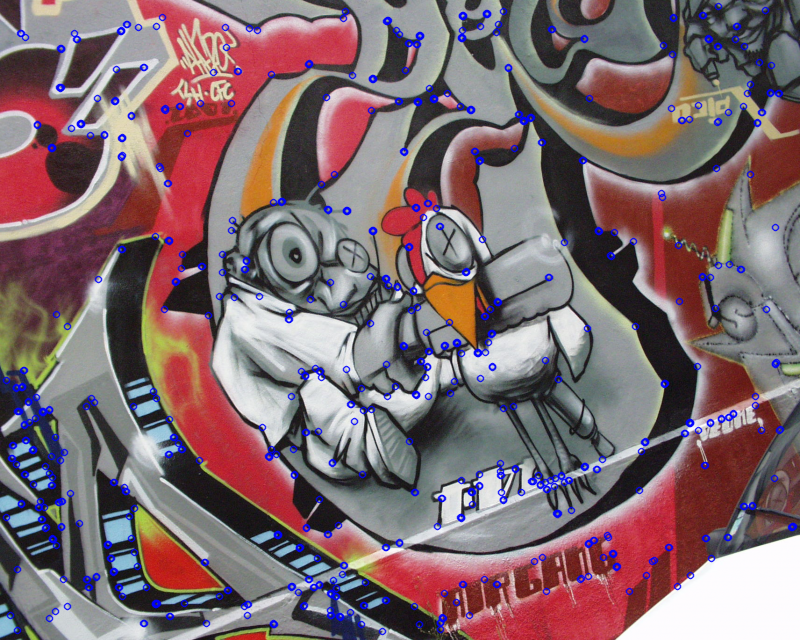}
         \caption{}
         \label{figure:viewpointb}
     \end{subfigure}
     \hfill
     \begin{subfigure}[b]{0.48\textwidth}
         \centering
         \includegraphics[width=\textwidth]{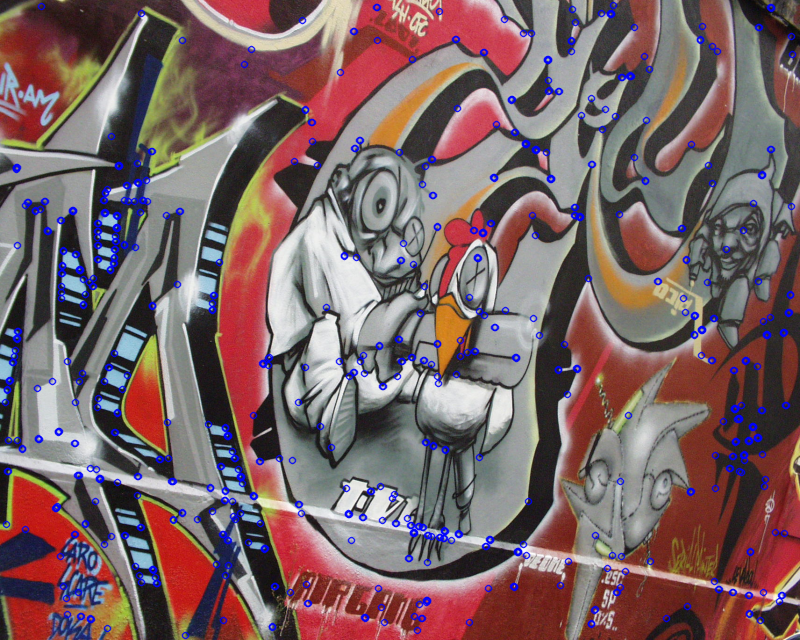}
         \caption{}
         \label{figure:viewpointc}
     \end{subfigure}
        \caption{Comparison of repeatability over viewpoint on the ``Graffiti" scene from the VGG dataset~\cite{mikolajczyk2005comparison}. (a) A comparison of our vertex detector with SIFT~\cite{lowe2004distinctive}, ORB~\cite{rublee2011orb}, and MSER~\cite{matas2004robust}; (b-d) Examples of our vertex detector's output in images captured at the frontal view and at viewing angles 20 and 30 degrees. Our repeatability is comparable to other popular interest point detectors.}
         \label{figure:viewpoint}
\end{figure}

\newpage

\section{Comparison with $\text{L}_0$ smoothing}

Figure~\ref{figure:l0_vs_ours} shows a comparison of our boundary-aware smoothing results with those obtained by $\text{L}_0$ smoothing~\cite{xu2011image}, increasing each method's structural smoothing parameter (denoted $\lambda_B$ for the field of junctions and $\lambda$ in~\cite{xu2011image}). Our method also has a parameter for selecting the spatial size of the details in the smoothed image (see for example the smoothed details inside the structure in the bottom example of Figure~\ref{figure:l0_vs_ours}).

\begin{figure*}[ht!]
\hspace{0.6cm}Smoother $\xrightarrow{\hspace*{14.2cm}}$\vspace{0.1cm}
 \captionsetup[subfigure]{labelformat=empty}
 \subfloat[]{
   \raisebox{0cm}{\rotatebox[origin=t]{90}{\hspace{1.3cm}Ours\hspace{2.3cm}$\text{L}_0$ smoothing\hspace{0.6cm}}}
}
 \subfloat[]{
   \includegraphics[width=0.97\linewidth]{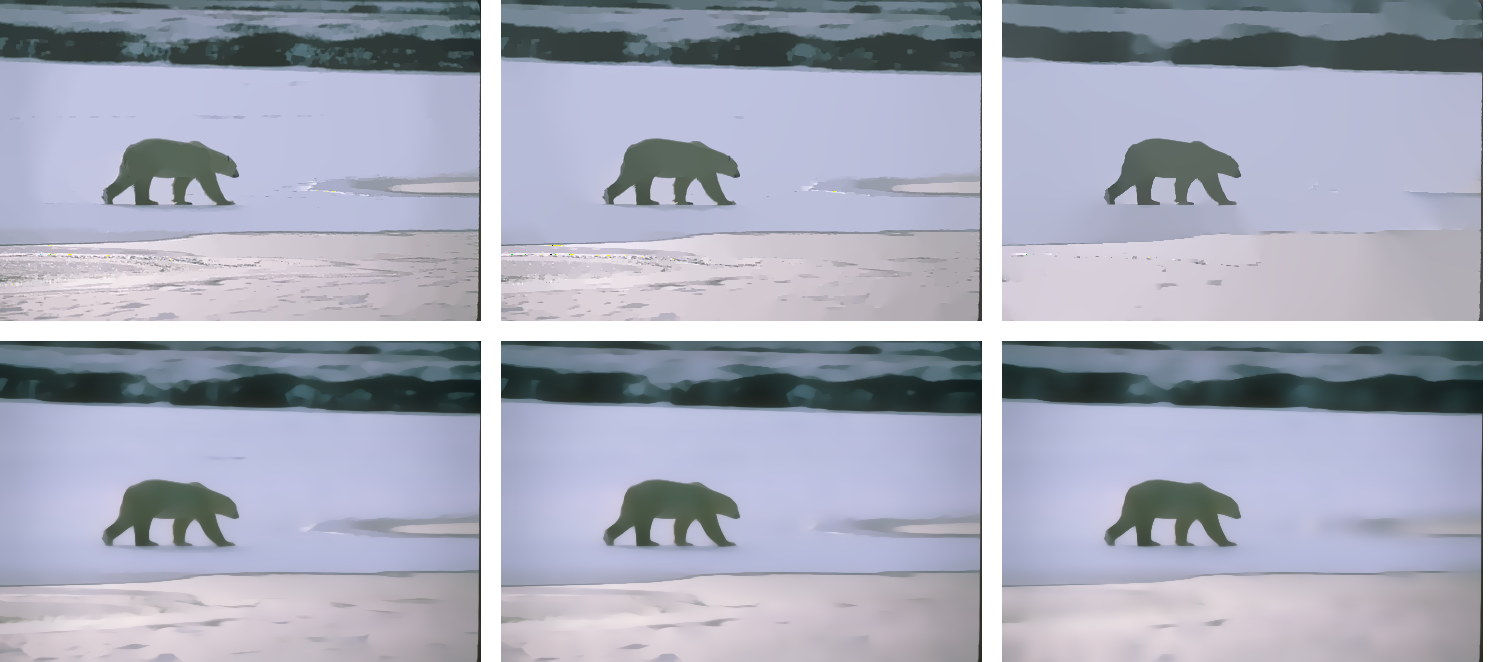}
}\\
 \subfloat[]{
   \raisebox{0cm}{\rotatebox[origin=t]{90}{\hspace{1.3cm}Ours\hspace{2.3cm}$\text{L}_0$ smoothing\hspace{0.6cm}}}
}
 \subfloat[]{
   \includegraphics[width=0.97\linewidth]{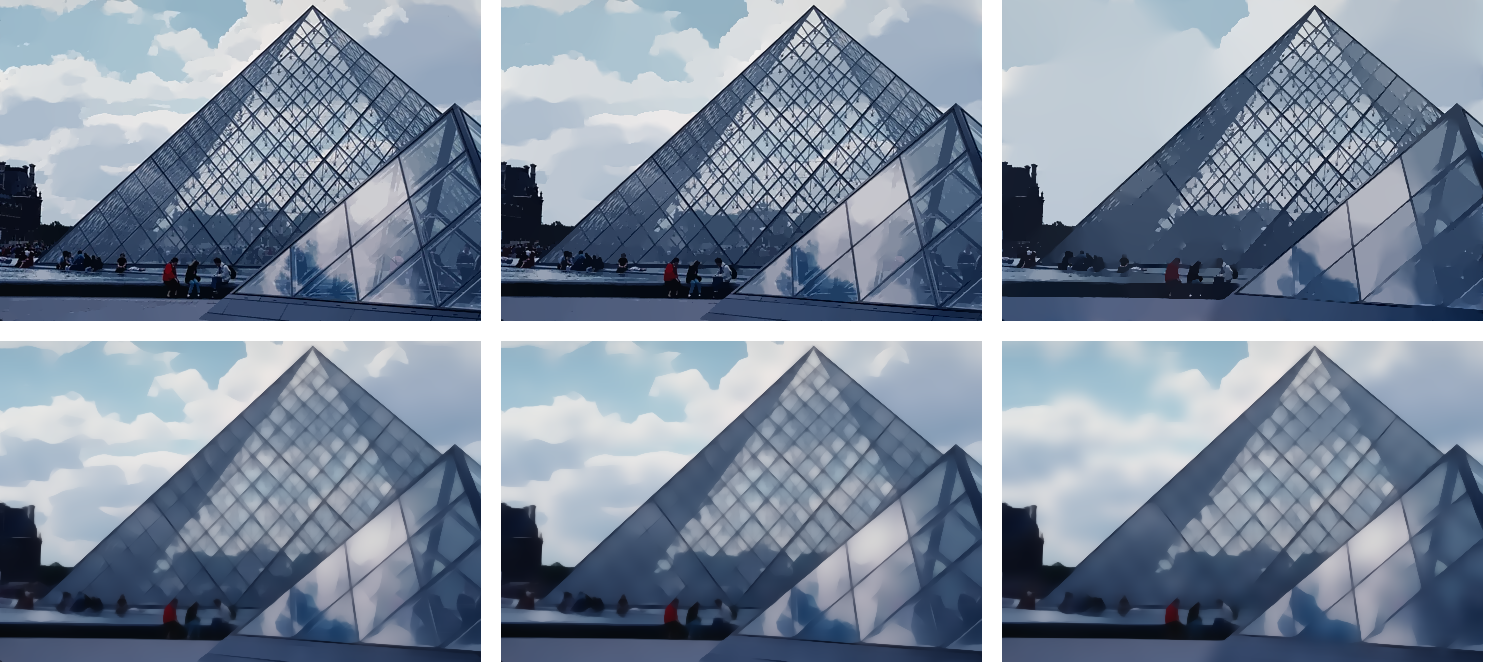}
}
\caption{The smoothed results obtained by $\text{L}_0$ smoothing (top row of each example) compared to our boundary-aware smoothing (bottom row). The smoothing parameter is increased from left to right for both methods.}
\label{figure:l0_vs_ours}
\end{figure*}

\newpage

\section{Additional results on photographs}

Figure~\ref{figure:bsd} shows additional results of boundary-aware smoothing and boundary detection on photographs, similar to Figure~6 from the main paper. Figure~\ref{figure:fish} shows additional results obtained by the field of junctions on a noisy photograph with increasing noise levels. Finally, Figure~\ref{figure:edge_sketch} shows an artistic effect obtained by combining the smoothed image with the boundary map using the field of junctions of a photograph.

\begin{figure}[H]
\begin{center}
\hspace{2.8cm}
\makebox[0pt]{Input}\hfill
\makebox[0pt]{Boundary-aware smoothing}\hfill
\makebox[0pt]{Boundaries}
\hspace{2.8cm}

\includegraphics[width=\linewidth]{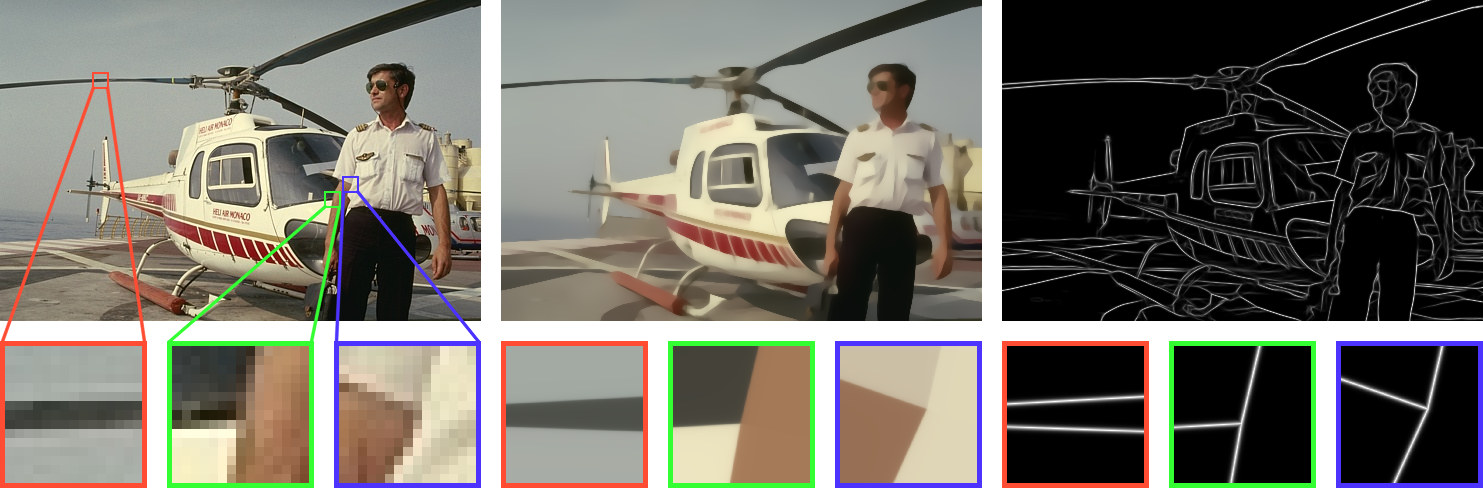} \\\vspace{15px}
\includegraphics[width=\linewidth]{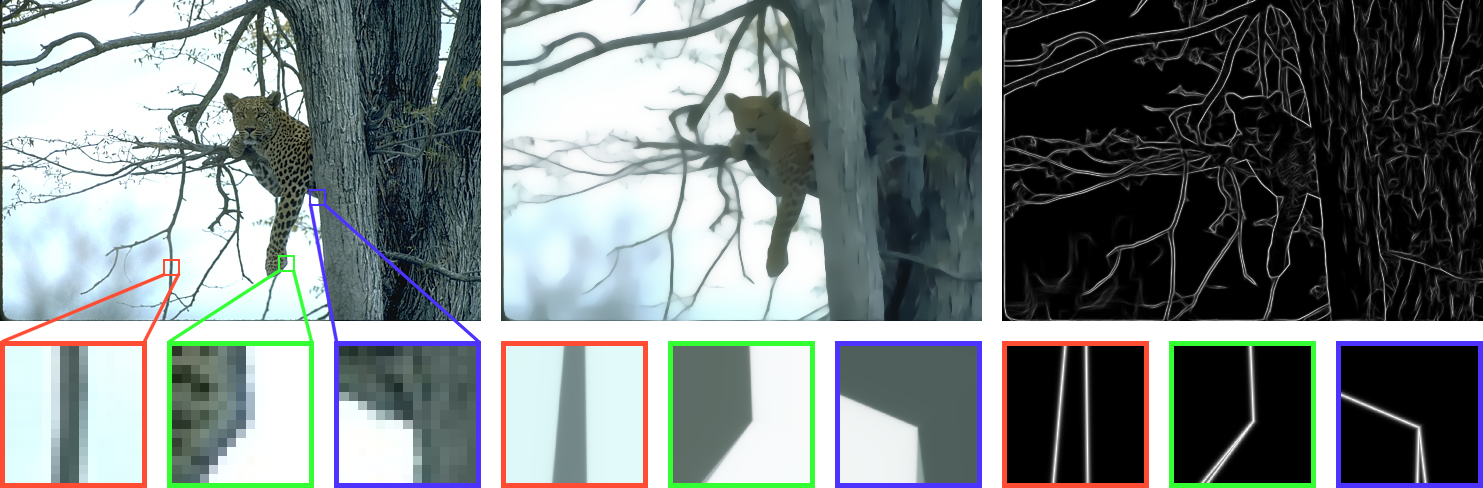} \\\vspace{15px}
\includegraphics[width=\linewidth]{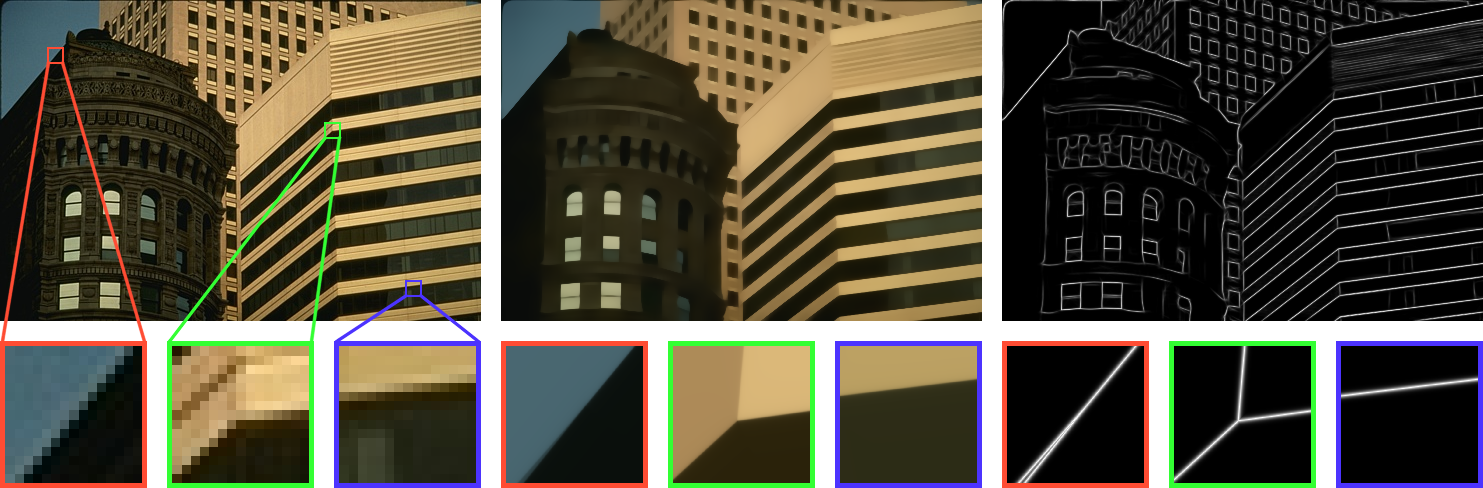}
\end{center}
\caption{Additional results on photographs.}
\label{figure:bsd}
\end{figure}

\newpage

\begin{figure}[H]
\hspace{2cm}
Input
\hspace{4.3cm}
Boundaries
\hspace{1.95cm}
Boundary-aware smoothing + junctions
\begin{center}
\includegraphics[width=0.33\linewidth]{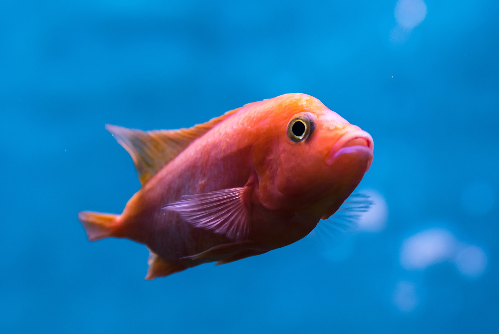}
\includegraphics[width=0.33\linewidth]{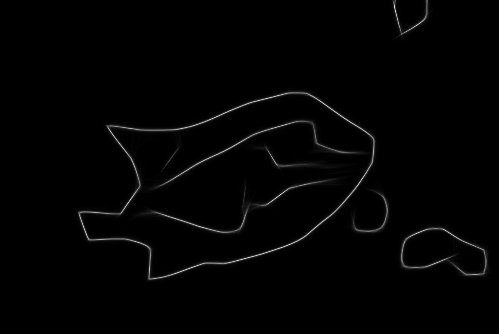}
\includegraphics[width=0.33\linewidth]{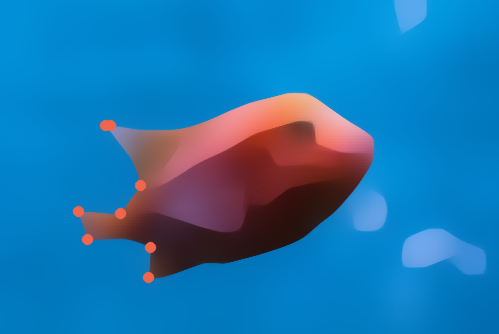}
\\\vspace{3px}
\includegraphics[width=0.33\linewidth]{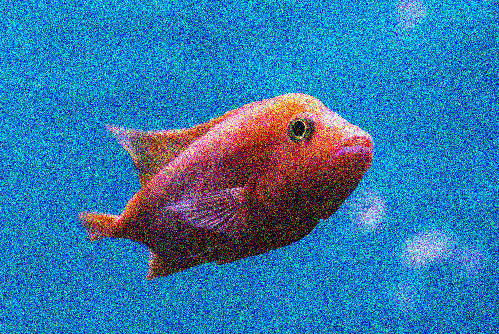}
\includegraphics[width=0.33\linewidth]{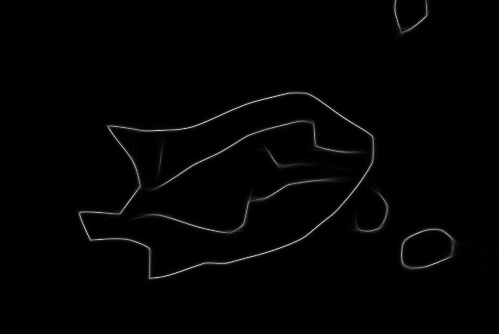}
\includegraphics[width=0.33\linewidth]{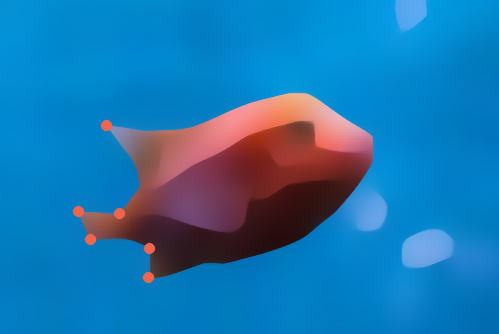}
\\\vspace{3px}
\includegraphics[width=0.33\linewidth]{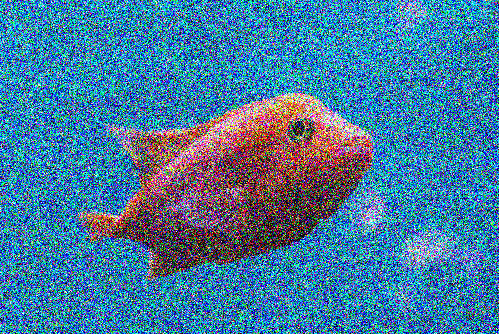}
\includegraphics[width=0.33\linewidth]{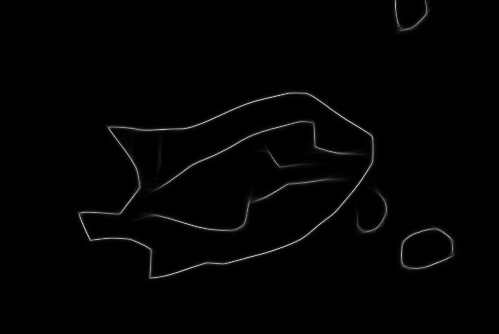}
\includegraphics[width=0.33\linewidth]{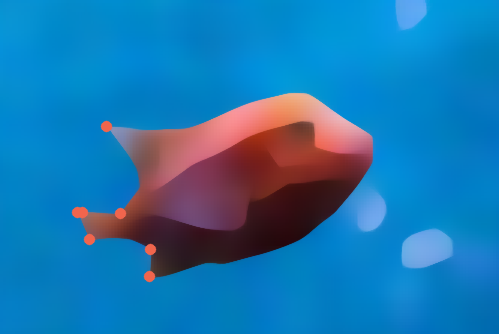}
\\\vspace{3px}
\includegraphics[width=0.33\linewidth]{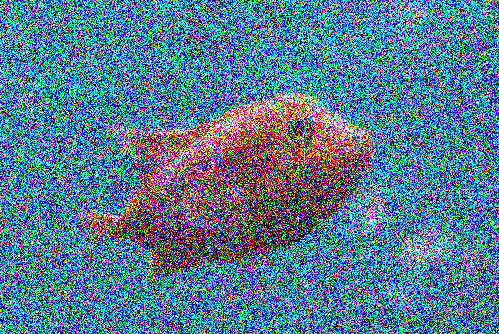}
\includegraphics[width=0.33\linewidth]{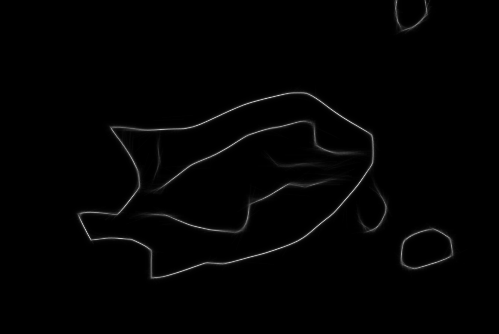}
\includegraphics[width=0.33\linewidth]{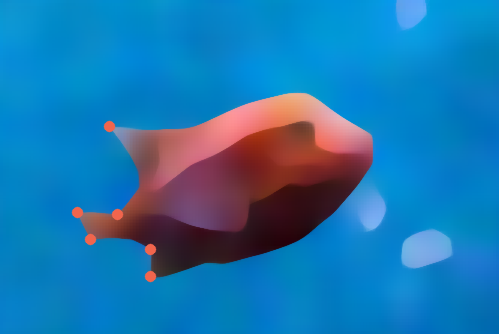}
\\\vspace{3px}
\includegraphics[width=0.33\linewidth]{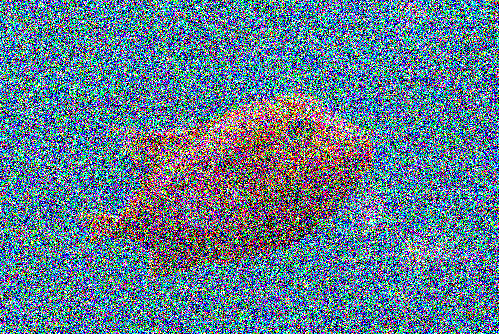}
\includegraphics[width=0.33\linewidth]{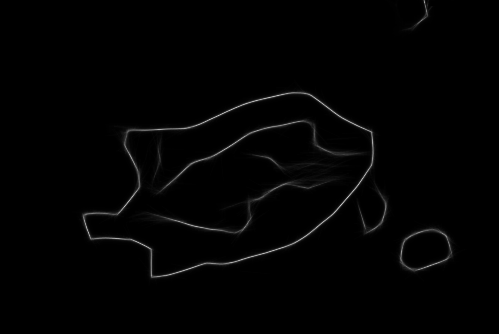}
\includegraphics[width=0.33\linewidth]{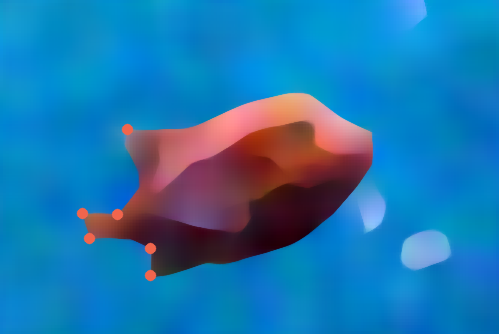}

\end{center}
\caption{Additional results on an increasingly noisy photograph.}
\label{figure:fish}
\end{figure}

\newpage

\begin{figure}[H]
\hspace{4cm}
\makebox[0pt]{Input}\hfill
\makebox[0pt]{Edited image}
\hspace{4cm}

\begin{center}
    \includegraphics[width=0.48\linewidth]{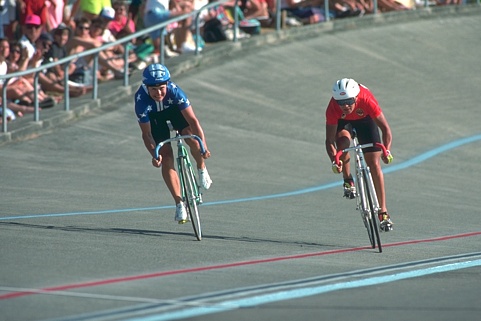}
    \includegraphics[width=0.48\linewidth]{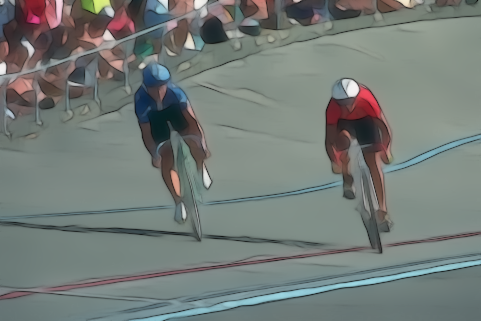}
\end{center}
\caption{Non-photorealistic editing by combining the smoothed image and the boundary map obtained from the field of junctions.}\label{figure:edge_sketch}
% The input image (top) is analyzed into its field of junctions and the edge-aware smoothing result is combined with the edge map to create an artistic effect (bottom). 
\end{figure}

%\input{sections/increasing_noise.tex}

%%%%%%%%
%%%%%%%%

\newpage

{\small
\changeBibPrefix{S}
\bibliographystyle{ieee_fullname}
\bibliography{sup_references}
}

\newpage